%% file: main.tex
\documentclass[letterpaper]{article} % DO NOT CHANGE THIS
\usepackage{aaai24}  % DO NOT CHANGE THIS
\usepackage{times}  % DO NOT CHANGE THIS
\usepackage{helvet}  % DO NOT CHANGE THIS
\usepackage{courier}  % DO NOT CHANGE THIS
\usepackage[hyphens]{url}  % DO NOT CHANGE THIS
\usepackage{graphicx} % DO NOT CHANGE THIS
\usepackage{natbib}  % DO NOT CHANGE THIS AND DO NOT ADD ANY OPTIONS TO IT
\usepackage{caption} % DO NOT CHANGE THIS AND DO NOT ADD ANY OPTIONS TO IT
\usepackage{algorithm}
\usepackage{algorithmic}

\nocopyright

\usepackage{newfloat}
\usepackage{listings}
\DeclareCaptionStyle{ruled}{labelfont=normalfont,labelsep=colon,strut=off} % DO NOT CHANGE THIS
\floatstyle{ruled}
\newfloat{listing}{tb}{lst}{}
\floatname{listing}{Listing}
\usepackage{graphicx}
\usepackage{amsmath}
\usepackage{amsthm}
\usepackage{amssymb}
\usepackage{dsfont}
\usepackage{booktabs}
\usepackage{algorithm}
\usepackage{algorithmic}
\usepackage[mathscr]{euscript}
\usepackage{xcolor}

\newcommand{\overbar}[1]{\mkern 1.5mu\overline{\mkern-1.5mu#1\mkern-1.5mu}\mkern 1.5mu}

\NewCommandCopy{\oldciteauthor}{\citeauthor}
\renewcommand{\citeauthor}[1]{\oldciteauthor{#1} (\citeyear{#1})}

\newtheorem{prop}{Proposition}

\newcommand{\Exp}[1]{\mathbb{E}\left[#1\right]}

\DeclareMathOperator*{\argmax}{argmax}
\DeclareMathOperator{\sign}{sgn}

\newtheorem{theorem}{Theorem}

\title{Lookahead Pathology in Monte-Carlo Tree Search}
\author{
    Khoi P. N. Nguyen\textsuperscript{\rm 1},
    Raghuram Ramanujan\textsuperscript{\rm 2}
}
\affiliations{
    \textsuperscript{\rm 1}Department of Computer Science, University of Texas at Dallas, Richardson, TX 75080\\
    \textsuperscript{\rm 2}Department of Mathematics and Computer Science, Davidson College, Davidson, NC 28036\\
    khoi.nguyen6@utdallas.edu, raramanujan@davidson.edu\\
    ~\\
    \emph{A shorter version appeared as a full paper in the Proceedings of the International Conference on Automated Planning and Scheduling (ICAPS) 2024.}
}

\begin{document}

\maketitle

\input{abstract}
\input{intro}

\input{background}

\input{model}
\input{results}
\input{concl}

\input{ack}

%% The file named.bst is a bibliography style file for BibTeX 0.99c
% \bibliographystyle{aaai24}
\bibliography{references}

\input{appendix}

\end{document}

%% file: abstract.tex
\begin{abstract}
Monte-Carlo Tree Search (MCTS) is a search paradigm that first found prominence with its success in the domain of computer Go. Early theoretical work established the soundness and convergence bounds for Upper Confidence bounds applied to Trees (UCT), the most popular instantiation of MCTS; however, there remain notable gaps in our understanding of how UCT behaves in practice. In this work, we address one such gap by considering the question of whether UCT can exhibit \emph{lookahead pathology} in adversarial settings --- a paradoxical phenomenon first observed in Minimax search where greater search effort leads to worse decision-making. We introduce a novel family of synthetic games that offer rich modeling possibilities while remaining amenable to mathematical analysis. Our theoretical and experimental results suggest that UCT is indeed susceptible to pathological behavior in a range of games drawn from this family.
\end{abstract}

%% file: intro.tex
\section{Introduction}
\label{sec:intro}
 
Monte-Carlo Tree Search (MCTS) is an online planning framework that originally found widespread use in game-playing applications \cite{crazystone,cadiaplayer,hex}, culminating in the spectacular success of AlphaGo \cite{alphago,alphagozero}. MCTS-based approaches have since been successfully adapted to a broad range of other domains, including combinatorial search and optimization \cite{uctmip,uctmaxsat}, malware analysis \cite{malware}, knowledge extraction \cite{ke}, and molecule synthesis \cite{chem1}.  

Despite these high-profile successes, however, there are still aspects of the algorithm that remain poorly understood. Early theoretical work by \citeauthor{mcts} introduced the Upper Confidence bounds applied to Trees (UCT) algorithm, now the most widely used variant of MCTS. Their work established that in the limit, UCT correctly identified the optimal action in sequential decision-making tasks, with the regret associated with choosing sub-optimal actions increasing at a logarithmic rate \cite{mcts}. \citeauthor{coquelin_munos}, however, showed that in the worst-case scenario, UCT's convergence could take time super-exponential in the depth of the tree. More recent work by \citeauthor{new_mcts} proposes a ``corrected'' UCT with better convergence properties.

In parallel, there has been a line of experimental work that has attempted to understand the reasons for UCT's success in practice and characterize the conditions under which it may fail. One such effort considered the impact of \emph{shallow traps} --- highly tactical positions in games like Chess that can be established as wins for the opponent with a relatively short proof tree --- and argued that UCT tended to misevaluate such positions \cite{RSS_traps1,RS_traps2}. \citeauthor{giga} considered the performance of UCT in a set of artificial games and pinpointed \emph{optimistic moves}, a notion similar to shallow traps, as a potential Achilles heel. \citeauthor{James_Konidaris_Rosman} studied the role of random playouts, a key step in the inner loop of the UCT algorithm, and concluded that the smoothness of the payoffs in the application domain determined the effectiveness of playouts. Our work adds to this body of empirical research, but is concerned with a question that has thus far not been investigated in the literature: \emph{can UCT behave pathologically?} 

The phenomenon of \emph{lookahead pathology} was first discovered and analyzed in the 1980s in the context of planning in two-player adversarial domains \cite{beal_pathology,nau_pathology,pearl_pathology}. Researchers found that in a family of synthetic board-splitting games, deeper Minimax searches counter-intuitively led to worse decision-making. In this paper, we present a novel family of abstract, two-player, perfect information games, inspired by the properties of real games such as Chess, in which UCT-style planning displays lookahead pathology under a wide range of conditions.

%% file: background.tex
\section{Background}
\label{sec:background}

\subsection{Monte-Carlo Tree Search}
\label{sec:mcts}

Consider a planning instance where an agent needs to determine the best action to take in a given state. An MCTS algorithm aims to solve this problem by iterating over the following steps to build a search tree.
\begin{itemize}
    \item \textbf{Selection:} Starting from the root node, we descend the tree by choosing an action at each level according to some policy $\pi$. UCT uses UCB1 \cite{ucb1}, an algorithm that optimally balances exploration and exploitation in the multi-armed bandit problem, as the selection policy. Specifically, at each state $s$, UCT selects the action $a = \pi(s)$ that maximizes the following upper confidence bound:
    $$ \pi(s) = \argmax_a \left( \overbar{Q}(T(s,a)) + c \cdot \sqrt{\frac{\log{n(s)}}{n(T(s,a))}} \right)$$
    Here, $T(s,a)$ is the transition function that returns the state that is reached from taking action $a$ in state $s$, $\overbar{Q}(s)$ is the current estimated utility of state $s$, and $n(s)$ is the visit count of state $s$. The constant $c$ is a tunable exploration parameter. In adversarial settings, the negamax transformation is applied to the UCB1 formula, to ensure that utilities are alternatingly maximized and minimized at successive levels of the search tree.
    \item \textbf{Evaluation:} The recursive descent of the search tree using $\pi$ ends when a node $s'$ that is previously unvisited, or that corresponds to a terminal state (i.e., one from which no further actions are possible), is reached. If $s'$ is non-terminal, then an estimate $R$ of its utility is calculated. This calculation may take the form of random playouts (i.e., the average outcome of pseudorandom completions of the game starting from $s'$), handcrafted heuristics, or the prediction of a learned estimator like a neural network. For terminal nodes, the true utility of the state is used as $R$ instead. The node $s'$ is then added to the search tree, so that the size of the search tree grows by one after each iteration.
    \item \textbf{Backpropagation:} Finally, the reward $R$ is used to update the visit counts and the utility estimates of each state $s$ that was encountered on the current iteration as follows:
    $$ \overbar{Q}(s) \leftarrow \frac{n(s)\overbar{Q}(s) + R(s)}{n(s) + 1} \qquad \quad n(s) \leftarrow n(s) + 1$$
    This update assigns to each state the average reward accumulated from every episode that passed through it.
\end{itemize}
We repeat the above steps until the designated computational budget is met; at that point, the agent selects the action $a = \argmax_{a'} \overbar{Q}(T(r, a'))$ to execute at the root node $r$.

\subsection{Lookahead Pathology}
\label{sec:pathology}

Searching deeper is generally believed to be more beneficial in planning domains. Indeed, advances in hardware that permitted machines to tractably build deeper Minimax search trees for Chess were a key reason behind the success of Deep Blue \cite{deep_blue}. However, there are settings in which this property is violated, such as the artificial P-games investigated by several researchers in the 1980s \cite{beal_pathology,nau_pathology,pearl_pathology} and which we describe in the following section. Over the years, many have attempted to explain the causes of pathology and why it is not encountered in real games like Chess. \citeauthor{path_survey} reconciled these different proposals and offered a unified explanation that focused on three factors: the branching factor of the game, the degree to which the game demonstrates \emph{local similarity} (a measure of the correlation in the utilities of nearby states in the game tree), and the \emph{granularity} of the heuristic function used (the number of distinct values that the heuristic takes on). They concluded that pathology was most pronounced in games with a high branching factor, low local similarity and low heuristic granularity \cite{path_survey}. They also found that while it was not a widespread issue in real games, certain sub-games of Chess and variations of Kalah could exhibit pathology. Finally, it is also worth noting that pathology has been observed in single-agent domains \cite{single-agent-pathology,RTS-pathology}, but our focus in this paper remains on the two-player setting where the phenomenon has been most extensively studied.

\subsection{Synthetic Game Tree Models}
\label{sec:tree-models}

\begin{figure*}[htb]
    \centering
    \includegraphics[width=0.33\textwidth]{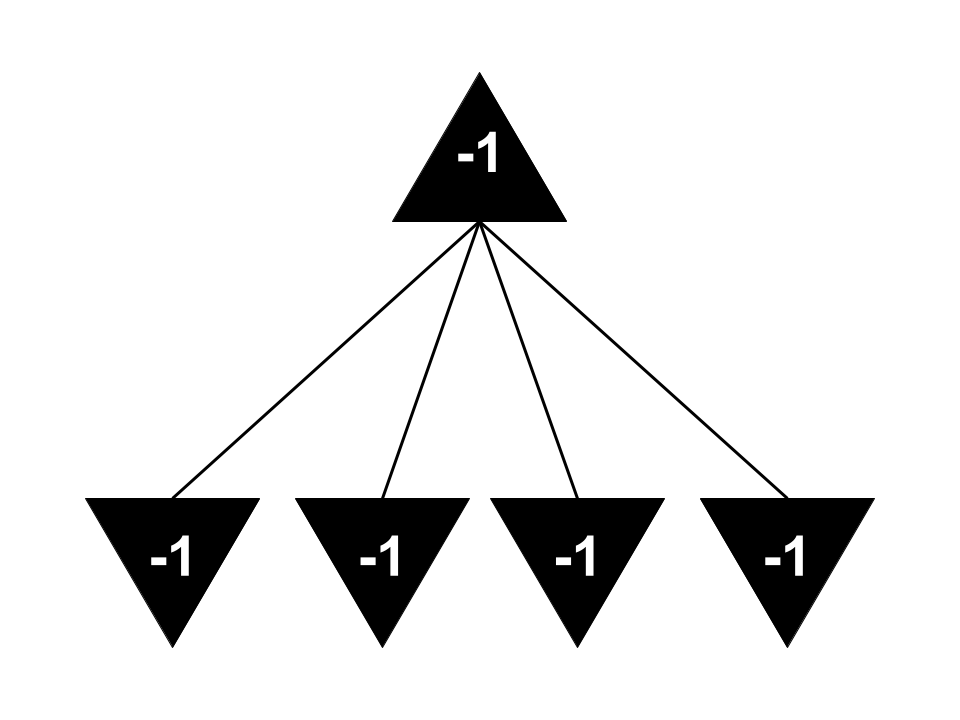}
    \qquad \qquad
    \includegraphics[width=0.33\textwidth]{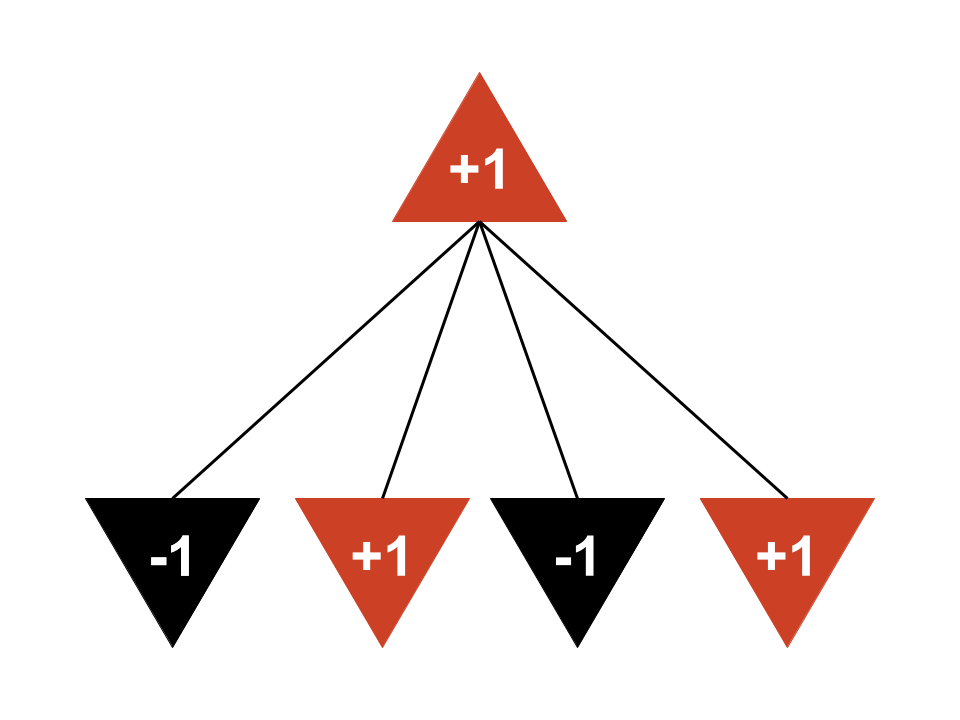}
    \caption{An example of a forced node (left) and a choice node (right). Upward-facing triangles represent maximizing nodes while downward-facing triangles represent minimizing nodes.}
    \label{fig:choice-forced}
\end{figure*}

There is a long tradition of using abstract, artificial games to empirically understand the behavior of search algorithms. The P-game model is a notable example, that constructs a game tree in a bottom-up fashion \cite{pearl80}. To create a P-game instance, the values of the leaves of the tree are carefully set to win/loss values \cite{pearl80}, though variants using real numbers instead have also been studied \cite{lustrek05}. The properties of the tree arise organically from
the distribution used to set the leaf node values. P-games were the subject of much interest in the 1980s, as the phenomenon of lookahead pathology was first discovered in the course of analyzing the behavior of Minimax search in this setting \cite{beal_pathology,nau_pathology}. While the model's relative simplicity allows for rigorous mathematical analysis, P-games also suffer from a couple of drawbacks. Firstly, computing the value of the game and the minimax value of the internal nodes requires search, and that all the leaf nodes of the tree be retained in memory, which restricts the size of the games that may be studied in practice. Secondly, the construction procedure only models a narrow class of games, namely, ones where the values of leaf nodes are independent of each other. 

Other researchers have proposed top-down models, where each internal node of the tree maintains some state information that is incrementally updated and passed down the tree. The value of a leaf node is then determined by a function of the path that was taken to reach it. For example, in the models studied by \citeauthor{nau83:probbackup} and \citeauthor{sk98}, values are assigned to the edges in the game tree and the utility of a leaf node is determined by the sum of the edge values on the path from the root node to the leaf. These models were used to demonstrate that correlations among sibling nodes were sufficient to eliminate lookahead pathology in Minimax. More recently, \citeauthor{socs11} used a similar top-down construction to evaluate the effectiveness of a distributed UCT implementation. However, search is still required to determine the true value of internal nodes in these models, thereby only allowing for the study of small games. \citeauthor{pvtrees} proposed \emph{prefix value trees} that extend the model of \oldciteauthor{sk98} by observing that the minimax value of nodes along a principal variation can never worsen for the player on move. Setting the values of nodes while obeying this constraint during top-down tree construction obviates the need for search, which allowed them to generate arbitrarily large games.

Finally, synthetic game tree models have also been used to study the behavior of MCTS algorithms like UCT. For example, \citeauthor{giga} used variations of Chess to identify the features of the search space that informed the success and failure of UCT. \citeauthor{RSS_critical_moves} studied P-games augmented with ``critical moves'' --- specific actions that an agent must get right at important junctures in the game to ensure victory. We refine this latter idea and incorporate it into a top-down model, which we present in the following section.

%% file: model.tex
\section{Critical Win-Loss Games}
\label{sec:crit-win-loss}

Our goal in this paper is to determine some sufficient conditions under which UCT exhibits lookahead pathology. To conduct this study, we seek a class of games that satisfy several properties. Firstly, we desire a model that permits us to construct arbitrarily large games to more thoroughly study the impact of tree depth on UCT's performance. We note that most of the game tree models discussed in Section \ref{sec:tree-models} do not meet this requirement. The one exception is the prefix value tree model of \oldciteauthor{pvtrees} that, however, fails a different test: the ability to construct games with parameterizable difficulty. Specifically, we find that prefix value games are too ``easy'' as evidenced by the fact that a na\"{i}ve planning algorithm that combines minimal lookahead with purely random playouts achieves perfect decision-making accuracy in this setting (see Appendix \ref{sec:pvtrees}). In this section, we describe \emph{critical win-loss games}, a new generative model of extensive-form games, that addresses both these shortcomings of existing models.

\subsection{Game Tree Model}
\label{sec:model}

\begin{figure*}[htb]
\centering
\includegraphics[width=0.3\textwidth]{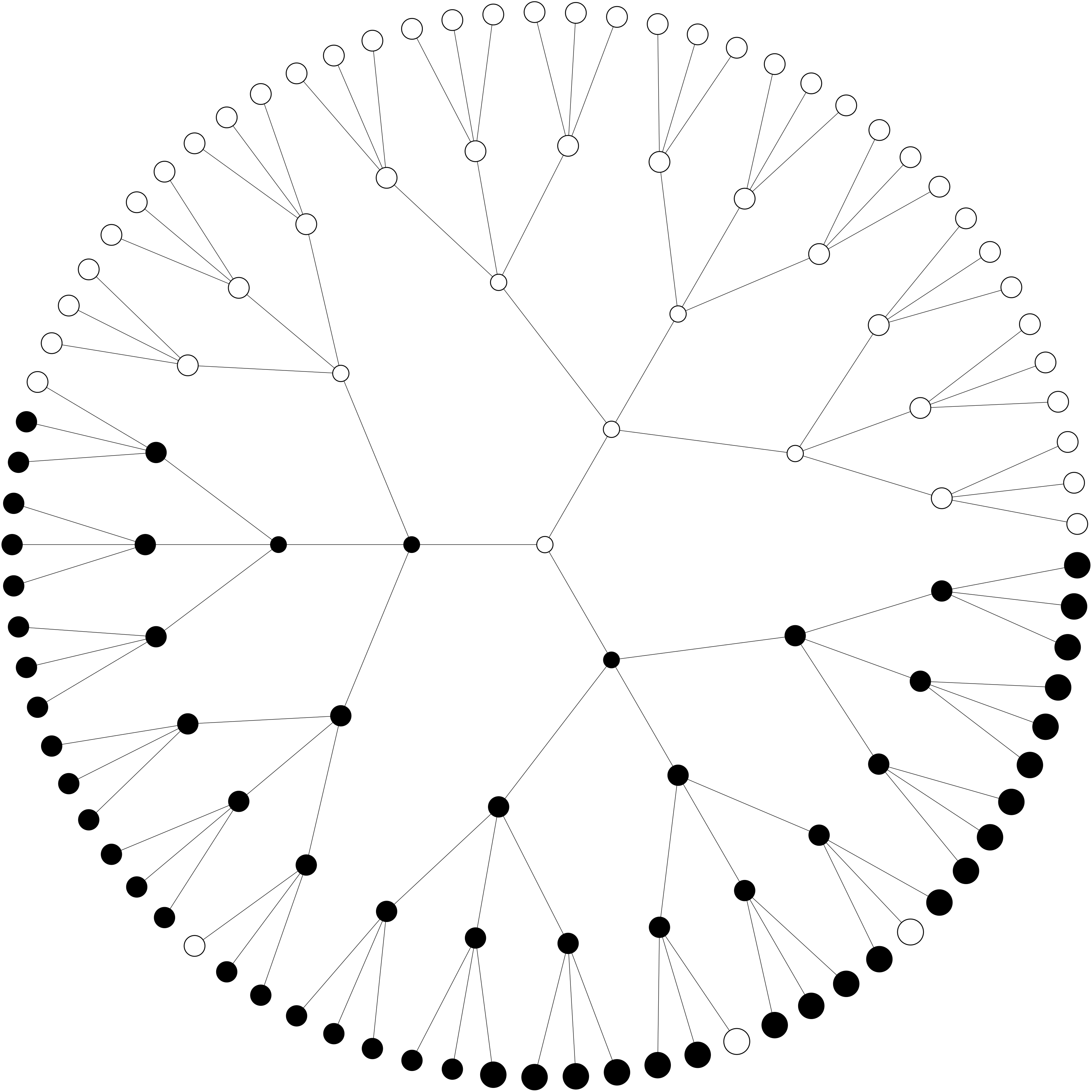}
\hfill
\includegraphics[width=0.3\textwidth]{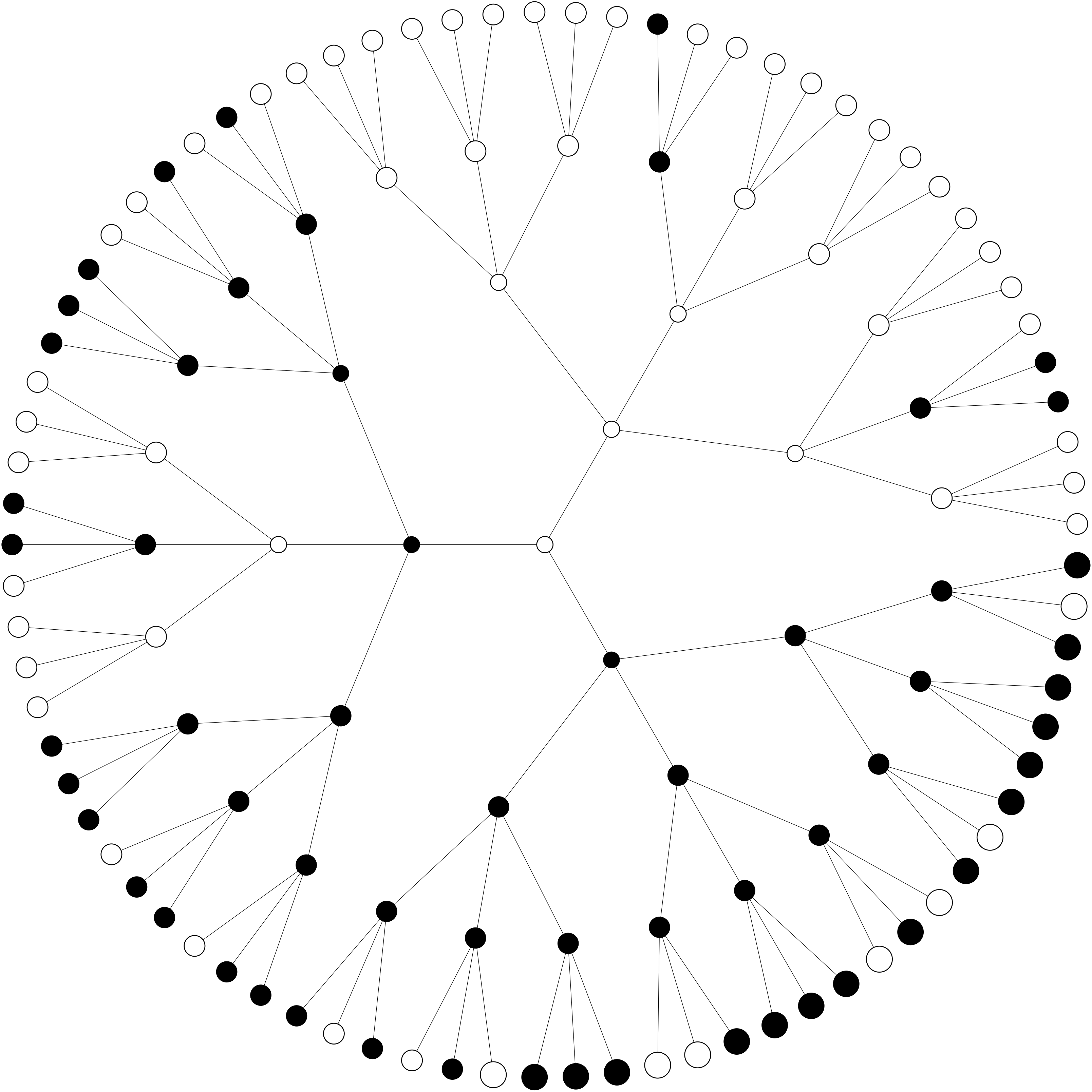}
\hfill
\includegraphics[width=0.3\textwidth]{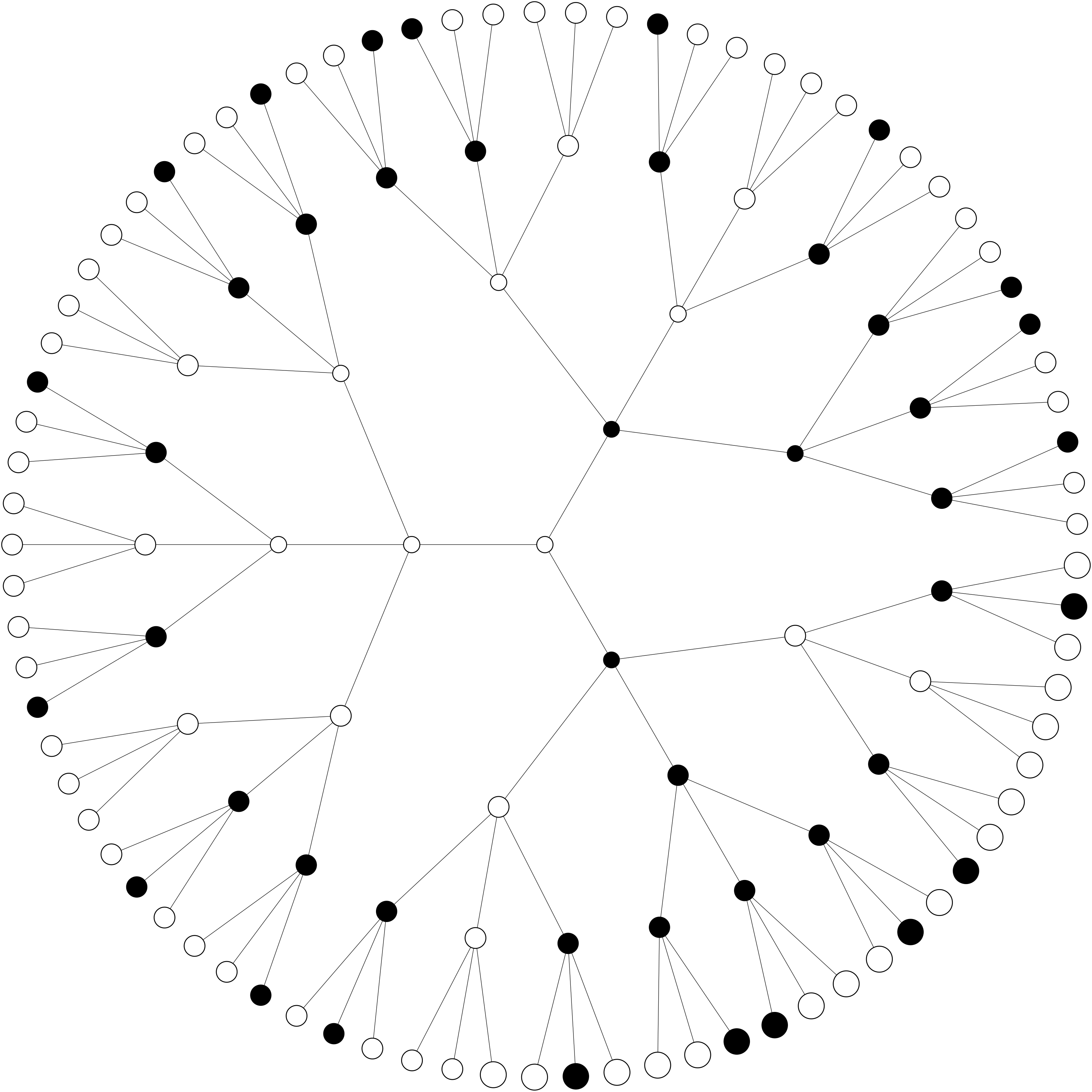}
\caption{Effect of critical rate ($\gamma$) on game tree structure. White nodes correspond to $+1$ positions, while black nodes correspond to $-1$ positions, with the root node in the center. The tree instances were generated with $\gamma=0.1$, $\gamma=0.5$, and $\gamma=1.0$, from left to right.}
\label{fig:trees}
\end{figure*}

Our model generates game trees in a top-down fashion, assigning each node a true utility of either $+1$ or $-1$. In principle, every state in real games like Chess or Go can be labeled in a similar fashion (ignoring the possibility of draws) with their true game-theoretic values. Thus, we do not lose any modeling capacity by limiting ourselves to just win-loss values. The true minimax value of a state imposes constraints on the values of its children as noted by \citeauthor{pvtrees} --- in our setting, this leads to two kinds of internal tree nodes. A \emph{forced node} is one with value $-1$ ($+1$) at a maximizing (minimizing) level. All the children of such a node are constrained to also be $-1$ ($+1$). A \emph{choice node}, on the other hand, is one with value $+1$ ($-1$) at a maximizing (minimizing) level. At least one child of such a node must have the same minimax value as its parent. Figure \ref{fig:choice-forced} presents examples of these concepts.

All the variation that is observed in the structures of different game trees is completely determined by what happens at choice nodes. As noted earlier, one child of a choice node must share the minimax value of its parent; the values of the remaining children are unconstrained. We introduce a parameter called the \emph{critical rate} ($\gamma$) that determines the values of these unconstrained children. We now describe our procedure for growing a critical win-loss tree rooted at a node $s$ with minimax value $v(s)$:
\begin{itemize}
    \item Let $S = \{s_1, s_2, \ldots, s_b\}$ denote the $b$ children of $s$.
    \item If $s$ is a forced node, then we set $v(s_1) = v(s_2) = \ldots = v(s_b) = v(s)$, and continue recursively growing each subtree.
    \item If $s$ is a choice node, then we pick an $s_i \in \{s_1, \ldots, s_b\}$ uniformly at random and set $v(s_i) = v(s)$, designating $s_i$ to be the child that corresponds to the optimal action choice at $s$. For each $s_j$ such that $j \neq i$, we set $v(s_j) = -v(s)$ with probability $\gamma$ and we set $v(s_j) = v(s)$ with probability $1 - \gamma$, before recursively continuing to grow each subtree.
\end{itemize}

We make several observations about the trees grown by this model. Firstly, one can apply the above growth procedure in a lazy manner, so that only those parts of the game tree that are actually reached by the search algorithm need to be explicitly generated. Thus, the size of the games is only limited by the amount of search effort we wish to expend. Secondly, the critical rate parameter serves as a proxy for game difficulty. At one extreme, if $\gamma=0$, then every child at every choice node has the same value as its parent --- in effect, there are no wrong moves for either player, and planning becomes trivial. At the other extreme, if $\gamma=1$, then every sub-optimal move at every choice node leads to a loss and the game becomes unforgiving. A single blunder at any stage of the game instantly hands the initiative to the opponent. Figure \ref{fig:trees} gives examples of game trees generated with different settings of $\gamma$. For the sake of simplicity, we focus on trees with a uniform branching factor $b$ in this study. 

%%% If space allows:
% Tree model can be ``standardized'' to have uniform depth, so games with trap states simply correspond to large subtrees with uniform leaf node evaluations. Same can be done with branching factor --- you can standardize to a fixed BF $b$ by simply repeating the optimal move until the total number of children at a given node is $b$.

%However, note that games with non-uniform branching factors can always be ``standardized'' to have some uniform $b$ by replicating the subtree rooted at the optimal action 

\subsection{Critical Rates in Real Games}
\label{sec:chess-cr}

Before proceeding, we pause to validate our model by measuring the critical rates of positions in Chess (see Appendix \ref{sec:othello-cr} for similar data on Othello). We begin by first sampling a large set of positions that are $p$ plies deep into the game. These samples are gathered using two different methods: 
\begin{itemize}
    \item \emph{Light playouts:} each side selects among the legal moves uniformly at random.
    \item \emph{Heavy playouts:} each side runs a 10-ply search using the Stockfish 13 Chess engine \cite{romstad2017stockfish}, which is freely available online under a GNU GPL v3.0 license, and then selects among the top-3 moves uniformly at random.
\end{itemize} 
We approximate $v(s)$ for these sampled states using deep Minimax searches.
Specifically, we use $v(s) \approx \sign{(\tilde v_d(s))}$, where $\tilde v_d(s)$ denotes the result of a $d$-ply Stockfish search. To compute the empirical critical rate $\tilde \gamma(s)$ for a particular choice node $s$, we begin by computing $\tilde v_{20}(s)$ and $\tilde v_{19}(s')$ for all the children $s'$ of $s$ and then calculate:
$$\tilde \gamma(s) = \frac{1}{b-1} \sum_{s'} \mathds{1}[\sign{(\tilde v_{19}(s'))} \neq \sign{(\tilde v_{20}(s))}]$$
Admittedly, using the outcome of a deep search as a stand-in for the true game theoretic value of a state is not ideal. However, strong Chess engines are routinely used in this manner as analysis tools by humans, and we thus believe this to be a reasonable approach. Figure \ref{fig:critical-rates} presents histograms of $\tilde \gamma$ data collected for $p = 10$ and $p = 36$, using both light and heavy playouts. Each histogram aggregates data over $\sim 20,000$ positions.

\begin{figure*}[htb]
    \centering
    \includegraphics[width=0.70\textwidth]{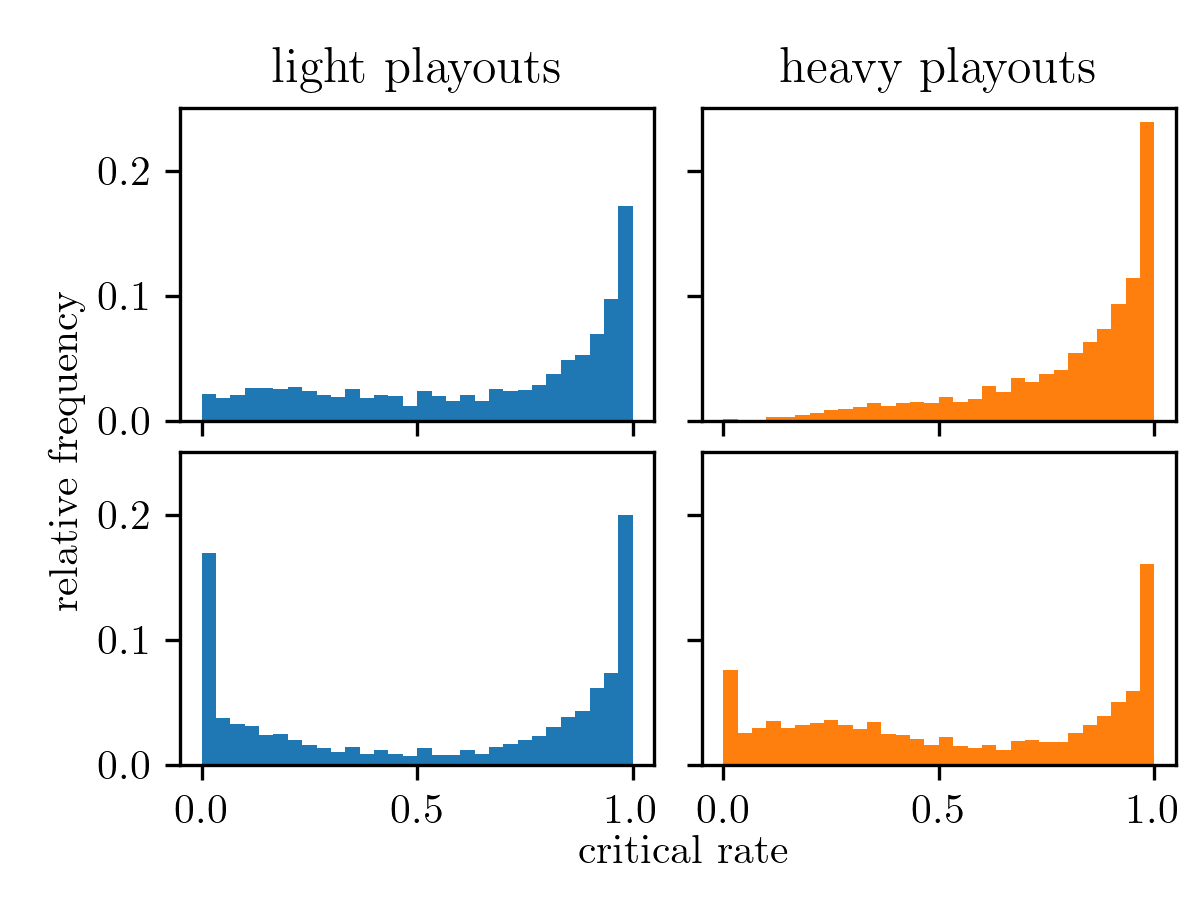}
    \caption{Histograms of empirical critical rates ($\tilde \gamma$) for Chess positions sampled $p=10$ (top row) and $p=36$ (bottom row) plies deep into the game. We sample the positions using both light playouts (left column) and heavy playouts (right column).}
    \label{fig:critical-rates}
\end{figure*}

We note that about $40$--$50\%$ of the positions sampled have $\tilde \gamma$ values higher than $0.9$, which is consistent with Chess's reputation for being a highly tactical game. We also see that the $\tilde \gamma$ values collected for Chess form a distribution that is non-stationary with respect to game progression, unlike in our proposed game tree model where $\gamma$ is fixed to be a constant. Nonetheless, we believe that this simplification in our modeling is reasonable: at deeper plies, the distribution of $\tilde \gamma$ becomes strikingly bimodal, with most of the mass accumulating in the ranges $[0.0, 0.1]$ and $[0.9, 1.0]$. This clustering means that one could partition Chess game tree into two very different kinds of subgames (with high and low $\gamma$), within each of which the critical rate remains within a narrow range.

\subsection{Heuristic Design}
\label{sec:heuristic}

Before we can run UCT search experiments on critical win-loss games, we need to resolve one more issue: \emph{how should UCT estimate the utility of non-terminal nodes?} One popular approach to constructing artificial heuristics is the additive noise model --- the heuristic estimate $h(s)$ for a node $s$ is computed as $h(s) = v(s) + \epsilon$, where $\epsilon$ is a random variable drawn from a standard distribution, like a Gaussian \cite{lustrek05,RSS_synthetic}. However, as we will see, static evaluations of positions in real games often follow complex distributions. 

\begin{figure*}[htb]
    \centering
    \includegraphics[width=0.48\textwidth]{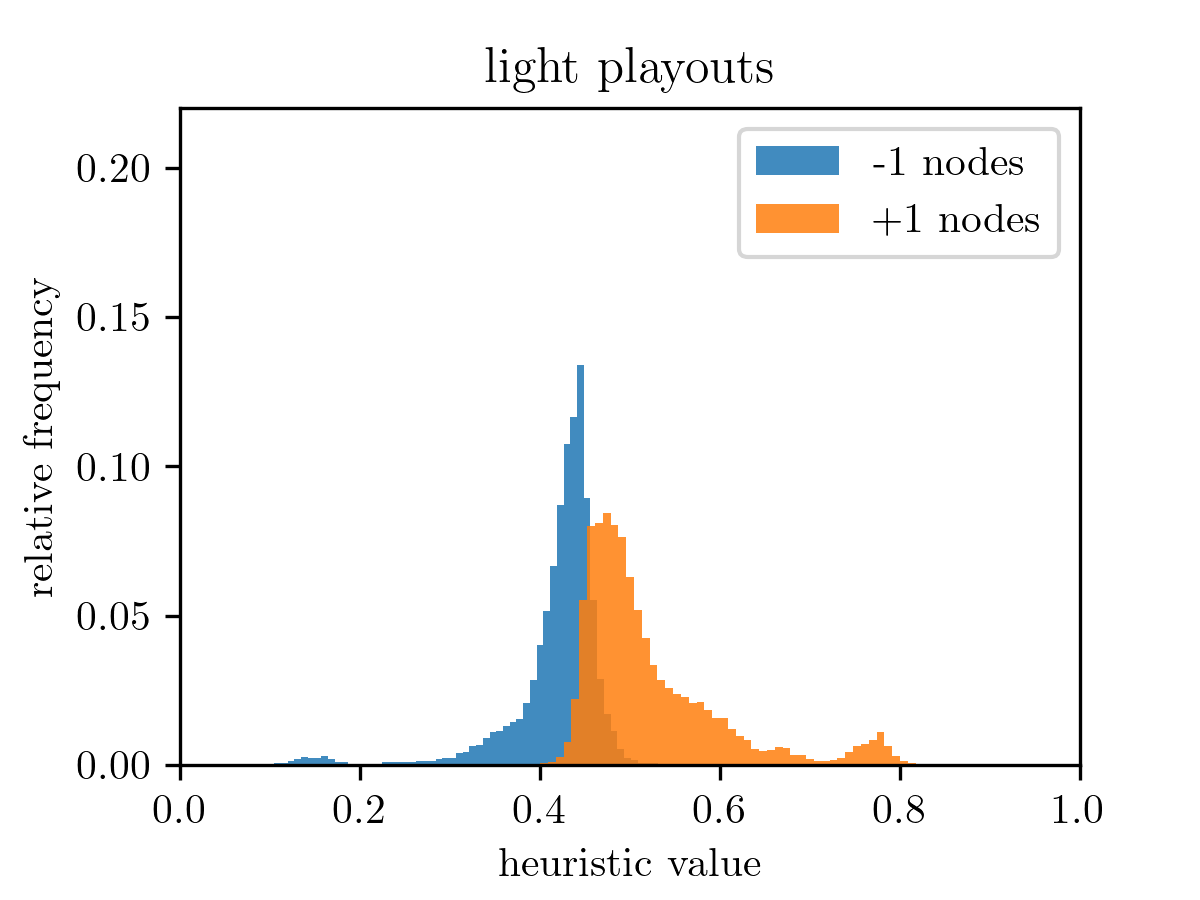}
    \includegraphics[width=0.48\textwidth]{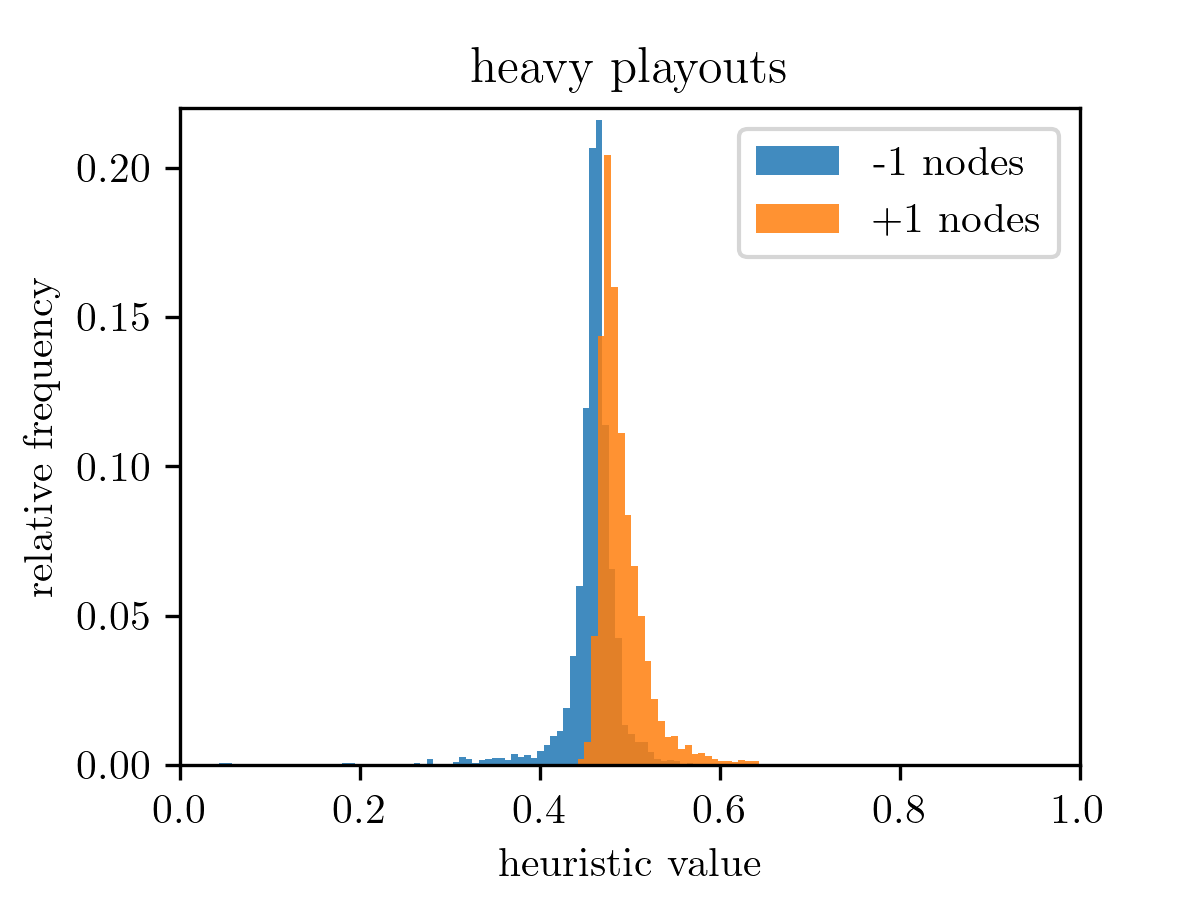}
    \caption{Distribution of Stockfish 13 static evaluations of $+1$ and $-1$ positions sampled $p=10$ plies deep into Chess. The positions are sampled using both light playouts (left) and heavy playouts (right).}
    \label{fig:heuristics}
\end{figure*}

To better understand the behavior of heuristic functions in real games, we once again turn to Chess and the Stockfish engine. We sample $\sim100,000$ positions each using light and heavy playouts for $p=10$. As before, we use $\sign{(\tilde v_{20}(s))}$ as a proxy for $v(s)$ for each sampled state $s$. We also compute $\tilde v_0(s)$ for each $s$, which we normalize to the range $[0, 1]$ --- this is the static evaluation of each $s$ without any lookahead. Figure \ref{fig:heuristics} presents histograms of $\tilde v_0(s)$, broken out by $v(s)$. A clearer separation between the orange and blue histograms (i.e., between the evaluations of $+1$ and $-1$ nodes) indicates that the heuristic is better at telling apart winning positions from losing ones. Indeed, the ideal heuristic would score every $+1$ position higher than every $-1$ position, thus ordering them perfectly.
We see in Figure \ref{fig:heuristics} that such clear sorting does not arise in practice in Chess, particularly for positions that are encountered with strong play. Moreover, the valuations assigned to positions do not follow a Gaussian distribution, and attempts to model them as such---for example, like in \cite{path_survey}---are perhaps too simplistic. However, these histograms also suggest an empirical method for generating heuristic valuations of nodes --- we can treat the histograms as probability density functions and sample from them. For example, to generate a heuristic estimate for a $-1$ node $s$ in our synthetic game, we can draw $h(s) \in [0, 1]$ according to the distribution described by one of the blue histograms in Figure \ref{fig:heuristics}. Of course, given the sensitivity of the shape of these histograms to the sampling parameters, it is natural to wonder \emph{which} histogram should be used. Rather than make an arbitrary choice, we run experiments using a diverse set of such histogram-based heuristics, generated from different choices of $p$, different playout sampling strategies, and different game domains. 

Additionally, we note that one can also use random playouts as heuristic evaluations, like in the original formulation of UCT. One advantage of our critical win-loss game tree model is that we can analytically characterize the density of $+1$ and $-1$ nodes at a depth $d$ from the root node, given a critical rate $\gamma$ and branching factor $b$. Specifically, for a tree rooted at a maximizing choice node, the density of $+1$ nodes at depths $2d$ and $2d+1$ (denoted as $f_{2d}$ and $f_{2d+1}$ respectively) are given by:
\begin{align}
f_{2d} = & ~k^{2d} + \frac{1 - k^{2d+2}}{1+k} \label{eq:even} \\
f_{2d+1} = & ~f_{2d} \cdot k \label{eq:odd}
\end{align}
where $k = 1 - \gamma~(1 - 1/b)$. We refer the reader to Appendix \ref{sec:leaf-nodes} for the relevant derivations. Access to these expressions means that we can cheaply simulate random playouts of depth $2d$: the outcome of a single playout ($\ell_1$) corresponds to sampling from the set $\{+1, -1\}$ with probabilities $f_{2d}$ and $1-f_{2d}$ respectively. For lower variance estimates, we can use the mean of this distribution instead, which would correspond to averaging the outcomes of a large number of playouts ($\ell_{\infty}$). In our experiments, we explore the efficacy of the heuristics $\ell_1$ and $\ell_{\infty}$ as well.

% \raghu{
% \begin{itemize}
    % \item Start by laying out desiderata for our model: want things to be top-down, because this allows games to be larger. For existing bottom-up constructions, we see some evidence of pathology but it dissipates right away, since We also want games to be challenging (trivial algorithms and heuristics shouldn't be effective). Idea of critical moves is related to idea of singular extensions first broached in the 80s in the precursors to Deep Blue \cite{sing_exts}: idea is to extend depth cutoff if one move looks better than all the others.
    % \item Every heuristic is really a weighted playout in disguise; so the question is what weighting makes sense?
    % \item Issue raised by Khoi: heuristic accuracy doesn't improve as we go deeper. Is this realistic? (One reason why this may be ok: we're trying to understand the performance of UCT in a specific setting where no leaf nodes in the game enter the horizon of the search tree, and we're interested in how UCT does in this setting.)
    % \item Related to above: there are no conditions where UCT seems to \textit{improve} --- it either flatlines or it gets worse. Is this realistic? One possible way out of the conundrum: this demonstrates that heuristics are \emph{good}, because with pure heuristics (playout based), we just flatline at poor performance, whereas with game-derived heuristics (game-based), we flat-line at high performance (in limited circumstances, we improve even, but often, we start at such a high floor that there's not much room for improvement).
%\end{itemize}
%}

%% file: results.tex
\section{Results}
\label{sec:results}

\subsection{Theoretical Analysis}
\label{sec:theory}

We begin with our main theoretical result and provide a sketch of the proof.\\
\begin{theorem}
    In a critical win-loss game with $\gamma=1.0$, UCT with a search budget of $N$ nodes will exhibit lookahead pathology for choices of the exploration parameter $c \geq \sqrt{\frac{N^3}{2 \log{N}}}$, even with access to a perfect heuristic.
    \label{thm:path-bound}
\end{theorem}
The key observation underpinning the result is that the densities of $+1$ nodes in both the optimal and sub-optimal subtrees rooted at a choice node (given by equations (\ref{eq:even}) and (\ref{eq:odd})) begin to approach the same value for large enough depths. This in turn suggests a way to lead UCT astray, namely to force UCT to over-explore so that it builds a search tree in a breadth-first manner. In such a scenario, the converging $+1$ node densities in the different subtrees, together with the averaging back-up mechanism in the algorithm, leaves UCT unable to tell apart the utilities of its different action choices. Notably, this happens even though we provide perfect node evaluations to UCT (i.e., the true minimax value of each node) --- the error arises purely due to the structural properties of the underlying game tree. All that remains to be done is to characterize the conditions under which this behavior can be induced, which is presented in detail in Appendix \ref{sec:path-bound}. Our experimental results suggest that the bound in Theorem \ref{thm:path-bound} can likely be tightened, since in practice, we often encounter pathology at much lower values for $c$ than expected. We further find that the pathology persists even when we relax the assumption that $\gamma=1.0$, as described in the following sections.

\subsection{Experimental Setup}
\label{sec:expts}

We now describe our experimental methodology for investigating pathology in UCT. Without loss of generality, we focus on games that are rooted at maximizing choice nodes (i.e., root node has a value of $+1$). We set the maximum game tree depth at $50$, which ensures that relatively few terminal nodes are encountered within the search horizon (other depths are explored in Appendix \ref{sec:depth}). We present results from a 4-factor experimental design: 2 choices of critical rate  ($\gamma$) $\times$ 3 choices of branching factor ($b$) $\times$ 2 choices of heuristic models $\times$ 5 choices of the UCT exploration constant ($c$). A larger set of results, exploring a wider range of these parameter settings, is presented in Appendices \ref{sec:depth}--\ref{sec:othello-path}. For each chosen parameterization, we generate $500$ synthetic games using our critical win-loss tree model. We run UCT with different computational budgets on each of these trees, as measured by the number of search iterations (i.e., the size of the UCT search tree). We define the \emph{decision accuracy} (denoted as $\delta_i$) to be the number of times that UCT chose the correct action at the root node, when run for $i$ iterations, averaged across the $500$ members of each tree family sharing the same parameter settings. Our primary performance metric is the \emph{pathology index} $\mathscr{P}_j$ defined as:
$$ \mathscr{P}_j = \frac{\delta_j}{\delta_{10}} $$
where $j \in \{10, 10^2, 10^3, 10^4, 10^5\}$. Values of $\mathscr{P}_j < 1$ indicate that additional search effort leads to worse outcomes (i.e., pathological behavior), while $\mathscr{P}_j > 1$ indicates that search is generally beneficial. We ran our experiments on an internal cluster of Intel Xeon Gold $5128$ $3.0$GHz CPUs with $512$G of RAM. We estimate that replicating the full set of results presented in this paper, with $96$ jobs running in parallel, would take about three weeks of compute time on a similar system. The code for reproducing our experiments is available at the following URL: \texttt{https://github.com/npnkhoi/mcts-lp}.

\begin{figure*}[tb]
    \centering
    \includegraphics[width=0.49\textwidth]{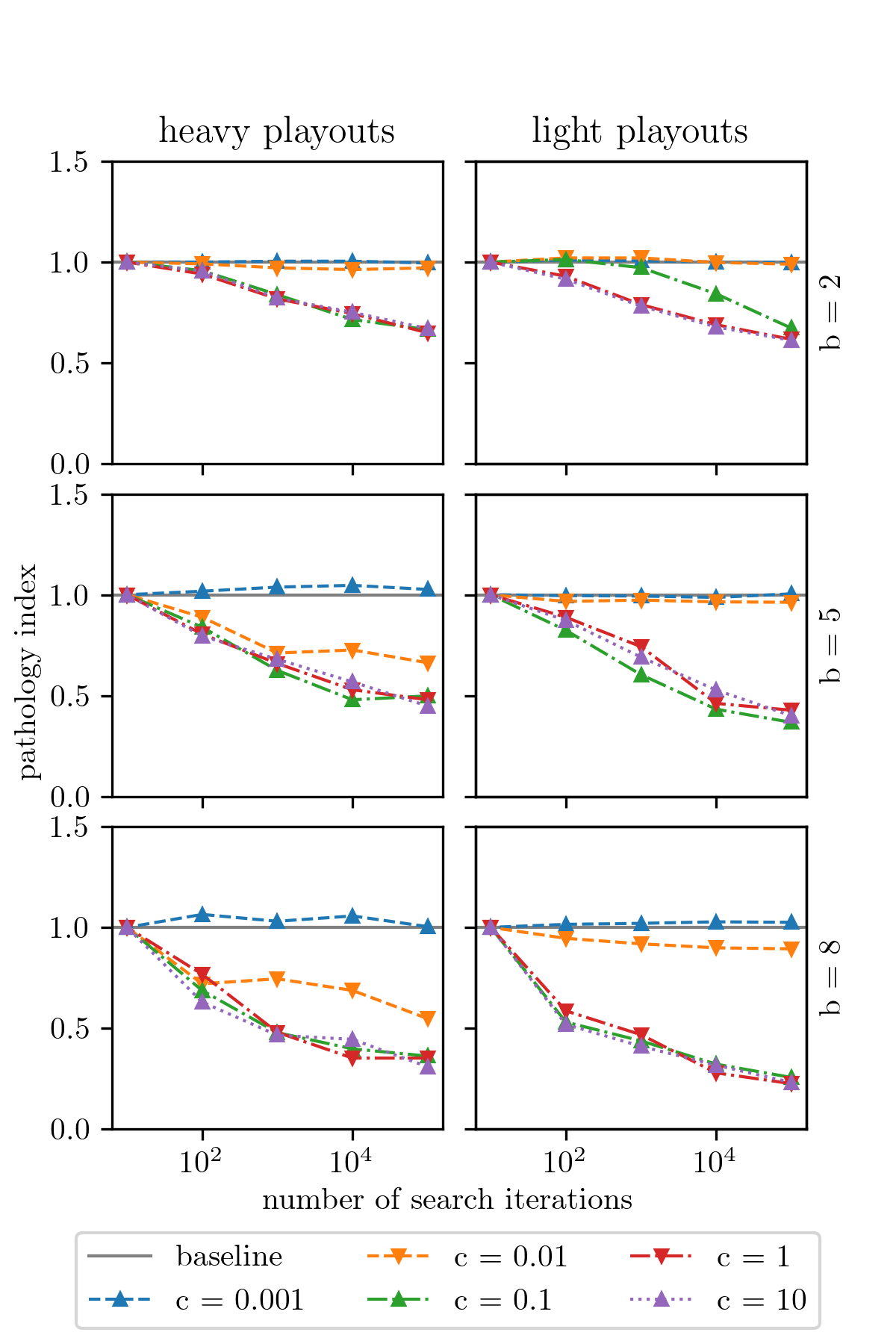}
    \includegraphics[width=0.49\textwidth]{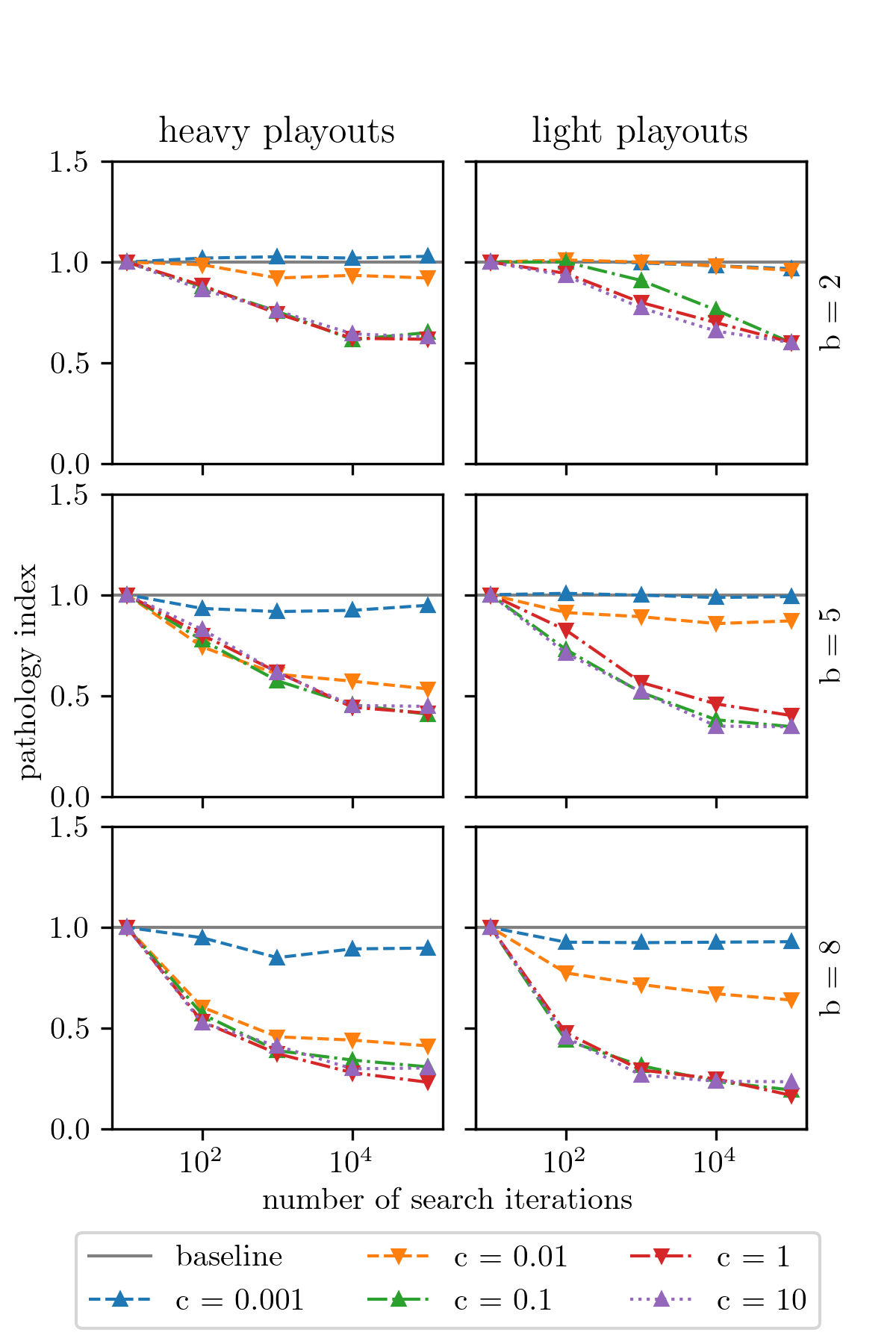}
    \caption{Measuring pathological behavior in UCT on critical win-loss games of depth $50$ with $\gamma=0.9$ (left) and $\gamma=1$ (right). The heuristic to guide UCT is constructed from histograms of Stockfish evaluations of positions sampled at depth $10$, using both light and heavy playouts. Each colored line corresponds to an instantiation of UCT with a different exploration constant. The $x$-axis is plotted on a log-scale.}
    \label{fig:uct-path}
\end{figure*}

\subsection{Discussion}
\label{sec:disc}

Figure \ref{fig:uct-path} presents our main results. Our chief finding is that the choice of $\gamma$ is the biggest determinant of pathological behavior in UCT. For games generated with $\gamma=1$, we find that UCT exhibits lookahead pathology \emph{regardless} of all other parameters --- the exploration constant, the branching factor, or how the heuristic is constructed. Appendices \ref{sec:depth} and \ref{sec:othello-path} confirm the robustness of this result using additional data collected for other game tree depths and for heuristics constructed using data collected from Othello. For smaller values of $\gamma$, the effect is not as strong and other factors begin to play a role. For example, with $\gamma=0.9$, pathological behavior is most apparent at higher branching factors and with more uniform exploration strategies (i.e., higher settings of $c$), consistent with Theorem \ref{thm:path-bound}. For $\gamma=0.5$, pathology is almost completely absent, regardless of other parameters (see Appendix \ref{sec:low-cr}).

Poor performance from UCT in a synthetic domain, particularly when it is forced to over-explore, may seem unsurprising at first glance --- so we will pause here to address these concerns, clearly elucidate the significance and novelty of these results, and to contextualize them better. Firstly, we note that lookahead pathology is distinct from poor planning performance. Practitioners routinely encounter domains where MCTS-style approaches simply fail to produce good results (for example, see \cite{RSS_traps1,RSS_critical_moves}), regardless of the size of the computational budget, but our work demonstrates a subtly different phenomenon. Namely, there are situations where UCT will initially make good decisions (when it is allowed to build a tree of size, say, 100 nodes), but that its performance will degrade as it is afforded more ``thinking'' time (on trees of size 10k nodes). Moreover, we have done so in a two-player setting, where winning strategies for a player need to be robust to any counter-move by the opponent. In practice, this means that there are an exponential number of winning leaf nodes for both players. Thus, demonstrating degenerate behavior using constructions such as those of \citeauthor{coquelin_munos}, where the optimal leaf node is strategically hidden in the midst of an exponential number of suboptimal nodes, are not possible. In fact, our model is symmetric; we use the same critical rate $\gamma$ for both players so that the game is equally (un)forgiving for both players, without stacking the deck against one player or the other. Seeing pathology arise under these conditions is thus more unexpected. We also note that more exploration has been proposed in the past as an antidote to undesirable behavior in UCT \cite{coquelin_munos}. Our results indicate that over-exploration may create new problems of its own.

%%%% Used to be in the old version, dropped in the interests of space.
% We also find that random playouts provide limited signal for guiding planning algorithms in this domain, even in modestly deep trees. This is because equations (\ref{eq:even}) and (\ref{eq:odd}) (and their duals for $-1$ nodes) which describe the density of $+1$/$-1$ nodes quickly approach their limiting distribution. In practice, this means that $+1$ and $-1$ nodes that are more than $\sim20$ plies removed from the end of the game receive identical evaluations, resulting in random decision-making at the root node in most cases.
For the sake of completeness, we also evaluate the performance of Minimax search with alpha-beta pruning on our critical win-loss games. These results are presented in Appendix \ref{sec:ab-path}. We find that Minimax is similarly susceptible to lookahead pathology in our setting, particularly in games where $\gamma \geq 0.7$. Games with high critical rates correspond to those with a high clustering factor $f$ \cite{clustering}, and thus, low \emph{local similarity} --- in such games, there is a lower degree of correlation among the true utilities of nodes that are near each other in the game tree, which has been identified as a key driver of pathology \cite{path_survey}. Our results are thus consistent with the findings of \citeauthor{path_survey}. One point of departure from their findings, however, is that pathology arises in our experiments even though we use a highly granular heuristic. In their work, \citeauthor{path_survey} create heuristic estimates for node utilities via an additive Gaussian noise model, whereas ours are derived from binning Stockfish's evaluation function and treating that as a distribution. This suggests that the manner in which heuristics are modeled may contribute to pathology, in addition to their resolution.

We conclude by considering the performance of UCT on the classical P-games devised by \citeauthor{pearl80}. As noted in Section \ref{sec:tree-models}, the bottom-up nature of a P-game's construction requires that we keep the entire game tree in memory during our experiments, which limits the size of the games we can study. Our results, presented in Appendix \ref{sec:pgames-path}, indicate that UCT demonstrates few signs of pathology in these games. Inducing pathological behavior in UCT thus appears to require large games with more intricate structural properties --- requirements which are satisfied by the critical win-loss games which we have introduced in this work.

% exhibiting pathology at $\gamma$ values as low as $0.5$. Indeed, our results indicate that lookahead pathology in Minimax search can arise even in games with low branching factor ($b=2$), modest local similarity ($\gamma = 0.5$) and with a highly granular heuristic , which runs counter to the findings of \citeauthor{path_survey} \cite{path_survey}. While reconciling these results is an important next step, we defer that investigation to future work as it is beyond the scope of this paper. 

% \cite{clustering} (definition of clustering factor $f$, but we adapt it to use Shannon Entropy rather than standard deviation to suit our win/loss node utilities)

\subsection{Broader Impacts}
\label{sec:broader-impacts}

This paper highlights a counter-intuitive failure mode for MCTS that deserves broader appreciation and recognition from researchers and practitioners. The fact that UCT has the potential to make worse decisions when given additional compute time means that the algorithm needs to be used with greater care. We recommend that users generate scaling plots such as those shown in Figure~\ref{fig:uct-path} to better understand whether UCT is well-behaved in their particular application domain, before wider deployment.

%% file: concl.tex
\section{Conclusions}

In this paper, we explored the question of whether MCTS algorithms like UCT could exhibit lookahead pathology --- an issue hitherto overlooked in the literature. Due to the shortcomings of existing synthetic game tree models, we introduced our own novel generative model for extensive-form games. We used these critical win-loss games as a vehicle for exploring search pathology in UCT and found it to be particularly pronounced in high critical rate regimes. Important avenues for follow-up work include generalizing the theoretical results presented in this paper to games where $\gamma \neq 1$ and deriving tighter bounds for the exploration parameter $c$, as well as investigating whether such pathologies emerge in real-world domains. It would also be interesting to study whether pathological behavior can be mitigated by choosing alternatives to the UCB1 bandit algorithm in the MCTS loop.

%% file: ack.tex
\section*{Acknowledgements}
\label{sec:ack}

The authors would like to thank Michael Blackmon for technical support, Corey Poff, Ashish Sabharwal and Bart Selman for useful discussions, and the anonymous reviewers for their feedback during the peer review process. KPNN's work was supported by the TPBank STEM Scholarship at Fulbright University Vietnam.

%% file: appendix.tex
\newpage

\onecolumn
\appendix
\begin{center}
    \textbf{\LARGE{Supplementary Material}}
\end{center}
\renewcommand{\thesection}{\Alph{section}}

\vspace{0.1in}

%%%%%%%%%%%%%%%%%%%%%%%%%%%%%%%%%%%%%%%%%%%%%%%%%%%%%%%%%%%%%%%%%%%%%%

\section{On Planning in Prefix Value Trees}
\label{sec:pvtrees}

In a prefix value (PV) tree, the values of nodes are drawn from the set of integers, with positive values representing wins for the maximizing player
(henceforth, Max) and the rest indicating wins for the minimizing
player (Min). For the sake of simplicity, we disallow draws. Let
$m(v)$ represent the minimax value of a node $v$. We grow the subtree
rooted at $v$ as follows:

\begin{itemize}
  \item Let $V = \{v_1, v_2, \ldots, v_b\}$ represent the set of
  children of $v$, corresponding to action choices $A = \{a_1, a_2,
  \ldots, a_b\}$.
  \item Pick an $a_i \in A$ uniformly at random --- this is designated
  to be the optimal action choice at $v$.
  \item Assign $m(v_i) = m(v)$. If Max is on move at $v$, then $m(v_j)
  = m(v) - k, \forall j \neq i$. If Min is on move $v$, then
  $m(v_j) = m(v) + k, \forall j \neq i$.
\end{itemize}
Here, $k$ is a constant that represents the cost incurred by the
player on move for taking a sub-optimal action. The depth of the tree is controlled by the parameter $d_{max}$ and a uniform branching factor of $b$ is assumed.

The PV tree model, is an attractive object of study as despite its
relative simplicity, it captures a very rich class of games.  However,
it has one major drawback: a simple 1-ply lookahead search, using the
average outcome of random playout trajectories as a heuristic, achieves very high decision accuracies. We now explore this phenomenon a little deeper.

Without loss of generality, we restrict our attention to trees
where Max is on move at the root node $n$. Moreover, we require that
$m(n)=1$ and that $n$ has exactly one optimal child --- this ensures
that our search algorithm is faced with a non-trivial decision at
the root node. While we focus on the case where $b = 2$ in what
follows, extending our results to higher branching factors is
straightforward. We denote the left and right children of $n$ by $l$
and $r$ respectively and assume that $l$ is the optimal move. Define
$S_d(v)$ to be the sum of the minimax values of the leaf nodes in the
subtree of depth $d$ rooted at node $v$. \\
\begin{prop}
  \label{prop:playouts-good}
  $S_d(l) - S_d(r) = 2^d$ for all $d \geq 0$.
\end{prop}
\begin{proof}
We proceed by induction on $d$. For the base case, $S_0(l) - S_0(r) =
1 - 0 = 2^0$ by definition. Assume the claim holds for $d = t, t \geq
0$, where $t$ is a Max level. Then, $S_{t+1}(l) - S_{t+1}(r) = (S_t(l)
+ (S_t(l) - k)) - (S_t(r) + (S_t(r) - k)) = 2(S_t(l)-S_t(r)) =
2\cdot2^d = 2^{d+1}$. A symmetric argument can be made for the case
where $t$ is a Min level.
\end{proof}
Define $P(v)$ to be the average outcome of random playouts performed
from the node $v$. If the subtree rooted at $v$ has uniform depth $d$,
then $\Exp{P(v)} = S_d(v) / 2^d$. An immediate consequence of
Proposition~\ref{prop:playouts-good} is that $\Exp{P(l)} - \Exp{P(r)}
= (S_d(l) - S_d(r)) / 2^d = 1$, for any depth $d$, i.e., in the limit,
the estimated utility of the optimal move $l$ at the root will always
be greater than that of $r$. In other words, the decision accuracy of
a 1-ply lookahead search informed by random playouts approaches $100\%$ with increasing number of playouts, independent of the depth of the tree. It is straightforward to extend this result to the case even when $k$ is a random variable, drawn uniformly at random from some set $\{1, \ldots, k_{max}\}$.\\

%%%%%%%%%%%%%%%%%%%%%%%%%%%%%%%%%%%%%%%%%%%%%%%%%%%%%%%%%%%%%%%%%%%%%%

\section{On the Distribution of Leaf Node Values in Critical Win-Loss Games}
\label{sec:leaf-nodes}

In this section, we analyze the density of $+1$ nodes in critical win-loss game as a function of its depth ($d$), branching factor ($b$) and critical rate ($\gamma$). We limit ourselves to the case where $b$ is uniform, $b \geq 2$ and $0 < \gamma \leq 1$.

Without loss of generality, we assume that the root is a maximizing choice node (i.e., has a minimax value of $+1$). Then, the expected number of its $+1$ children is: 
$$1 + (1 - \gamma)\cdot (b - 1) = b + \gamma - b\gamma$$
We denote the expected \emph{density} of these $+1$ children (among the possible $b$ children) by $k$, where:
\begin{align}
k=\frac{b + \gamma - b\gamma}{b} = 1 - \gamma + \frac{\gamma}{b} \label{eq:k}
\end{align}
We note that by a symmetric argument, this expression also captures the density of the $-1$ children of a minimizing choice node.

% \begin{theorem}[Same-minimax density]
% \label{Same-minimax density}
%     From a choice node with the Minimax value of $M$, generate randomly its children. Among these children, the expected density of the nodes with the Minimax value of $M$ is $1 - c + \frac{c}{b}$.
% \end{theorem}

Now consider a critical win-loss game instance rooted at a maximizing choice node. Let $f_n$ denote the $+1$ density at depth $n$ in this tree. We will calculate a closed-form expression for $f_n$. When $n$ is even (i.e., Max is on move), we have the following recursive relationship due to equation (\ref{eq:k}): 
$$f_{n+1} = f_{n}\cdot\left(1 - \gamma + \frac{\gamma}{b}\right)$$ 
Substituting $2d$ for $n$, we have:
\begin{align}
    f_{2d+1} = f_{2d} \cdot k \label{eq:f-odd}
\end{align}

When $n$ is odd (i.e., Min is on move), the choice nodes have value $-1$. Once again using equation (\ref{eq:k}), we derive the following recursive equation for $-1$ nodes:
\begin{align*}
    1 - f_{n+1} &= (1 - f_{n})\cdot\left(1-\gamma+\frac{\gamma}{b} \right) \\
    f_{n+1} &= 1 - (1 - f_{n})\left(1-\gamma+\frac{\gamma}{b}\right)\\
    f_{n+1} &= f_{n} \cdot k + 1 - k \\
\end{align*}
Substituting $2d+1$ for $n$, we have:
\begin{align}
    f_{2d+2} &= f_{2d+1} \cdot k + 1 - k \label{eq:f-even}
\end{align}

From equations (\ref{eq:f-odd}) and (\ref{eq:f-even}), we have:
\begin{align*}
    f_{2d+2} &= f_{2d} \cdot k^2 + (1 - k)
\end{align*}
By induction, we can then derive the following non-recurrent formula for $f_{2d}$:
\begin{eqnarray*}
    f_{2d} = f_0 \cdot (k^2) ^ d + (1 - k) \cdot \frac{1 - (k^2) ^ {d + 1}}{1 - k^2}
\end{eqnarray*}
where $f_0=1$. Simplifying, we have:
$$ f_{2d} = k^{2d} + \frac{1 - k^{2d+2}}{1+k} $$

If we now allow $d \to \infty$, we have:
\begin{align}
\lim_{d\to\infty}f_{2d} = \frac{1-k}{1-k^2} = \frac{1}{1+k} = \frac{1}{2 - \gamma + \frac{\gamma}{b}}
\label{eq:even-limit}
\end{align}
and:
\begin{align}
\lim_{d\to\infty}f_{2d+1} = \lim_{d\to\infty}f_{2d} \cdot k = \frac{k}{1+k} = \frac{1 - \gamma + \frac{\gamma}{b}}{2 - \gamma + \frac{\gamma}{b}}
\label{eq:odd-limit}
\end{align}
~\\

%%%%%%%%%%%%%%%%%%%%%%%%%%%%%%%%%%%%%%%%%%%%%%%%%%%%%%%%%%%%%%%%%%%%%%

\section{Deriving Bounds on Pathological Behavior}
\label{sec:path-bound}

We provide a proof of Theorem~\ref{thm:path-bound} from Section~\ref{sec:theory}, which is restated below.
{
\addtocounter{theorem}{-1}  % bit of a hack to get Thm number to line up
\begin{theorem}
    In a critical win-loss game with $\gamma=1.0$, UCT with a search budget of $N$ nodes will exhibit lookahead pathology for choices of the exploration parameter $c \geq \sqrt{\frac{N^3}{2 \log{N}}}$, even with access to a perfect heuristic.
\end{theorem}
}
% \begin{theorem}
%     With $c \geq \sqrt{\frac{N^3}{2\log N}}$, in the first $N$ iterations, UCT will traverse the game tree in a BFS-like manner, i.e. visiting all nodes in the upper layer before visiting the next.
% \end{theorem}
\begin{proof}
    The proof consists of two parts. First, we argue that with an appropriately chosen value for the exploration constant $c$, UCT will build a balanced search tree in a breadth-first fashion. Then, we show that such a tree building strategy will cause UCT's decision accuracy to devolve to random guessing with increased search effort. Consider a node $p$ with $b$ children. Let $a_1$ and $a_2$ denote two of these children such that $n(a_1) < n(a_2)$. Our aim is to find a value for the exploration parameter $c$ such that the UCB1 formula will prioritize visiting $a_1$ over $a_2$, regardless of the difference in their (bounded) utility estimates. This amounts to UCT building a search tree in a breadth-first fashion.\\
    
    Without loss of generality, assume $p$ is a maximizing node. For UCT to visit $a_1$ before $a_2$ on the next iteration, we must have:
    $$ \overbar{Q}(a_1) + c\sqrt{\frac{\log n(p)}{n(a_1)}} > \overbar{Q}(a_2) + c\sqrt{\frac{\log n(p)}{n(a_2)}} $$
    
    Rearranging terms, this is equivalent to the condition:
    \begin{align}
    c > \frac{\overbar{Q}(a_2) - \overbar{Q}(a_1)}{\sqrt{\log n(p)}\left(\sqrt{\frac{1}{n(a_1)}} - \sqrt{\frac{1}{n(a_2)}} \right)}
    \label{eq:bound}
    \end{align}
    
    We will now bound the right-hand side of equation~\ref{eq:bound} in terms of our search budget $N$. Firstly, since the heuristic estimates of nodes are bounded by $[0, 1]$, we know that $\overbar{Q}(a) \in [0, 1]$ for any node $a$. We can therefore bound the numerator of equation (\ref{eq:bound}) from above as:
    \begin{align}
    \overbar{Q}(a_2) - \overbar{Q}(a_1) \leq 1
    \label{eq:numerator}
    \end{align}
    
    Now we turn our attention to the denominator $D$ of equation (\ref{eq:bound}). We have:
    \begin{align*}
        D & = \sqrt{\log n(p)}\left(\sqrt{\frac{1}{n(a_1)}} - \sqrt{\frac{1}{n(a_2)}} \right) \\
        & = \sqrt{\log{n(p)}} \left[ \frac{n(a_2)-n(a_1)}{\sqrt{n(a_1)n(a_2)}(\sqrt{n(a_1)} + \sqrt{n(a_2}))} \right] \\
        & \geq \sqrt{\log{n(p)}} \left[ \frac{1}{\sqrt{n(a_1)n(a_2)}(\sqrt{n(a_1)} + \sqrt{n(a_2}))} \right] &\text{since~} n(a_2) > n(a_1) \\
        & \geq \sqrt{\log{n(p)}} \left[ \frac{1}{\frac{n(a_1)+n(a_2)}{2}(\sqrt{n(a_1)} + \sqrt{n(a_2}))} \right] & \text{by the AM-GM inequality}\\
        & \geq \sqrt{\log{n(p)}} \left[ \frac{1}{\frac{n(a_1)+n(a_2)}{2}\sqrt{2(n(a_1) + n(a_2))}} \right] & \text{since~} \sqrt{x}+\sqrt{y} \leq \sqrt{2(x+y)}\\
        & \geq \sqrt{\log{n(p)}} \left[ \frac{1}{\frac{n(p)}{2}\sqrt{2 \cdot n(p)}} \right] & \text{since~} n(a_1) + n(a_2) \leq n(p)\\
        & = \sqrt{\frac{2 \log{n(p)}}{n(p)^3}}\\
        & \geq \sqrt{\frac{2 \log{N}}{N^3}} %& \text{since~} \frac{d}{dx} \left( \frac{\log x}{x^3} \right) = \frac{3 \log x - 1}{x^4} > 0 \text{~for sufficiently large~} x\\
    \end{align*}

    Combining this bound with equation (\ref{eq:numerator}), we conclude that:
    $$ \frac{\overbar{Q}(a_2) - \overbar{Q}(a_1)}{\sqrt{\log n(p)}\left(\sqrt{\frac{1}{n(a_1)}} - \sqrt{\frac{1}{n(a_2)}} \right)} \leq \sqrt{\frac{N^3}{2 \log {N}}}$$
    Thus, choosing a value for the exploration constant $c$ that is larger than this quantity, as per equation (\ref{eq:bound}), will force UCT to build a search tree in a breadth-first fashion.\\

    We now conclude by arguing why such a node expansion strategy will lead to lookahead pathology. Without loss of generality, consider a game tree rooted at a minimizing choice node $p$ (i.e., a minimizing node with minimax value $-1$). Let $s^*$ denote an optimal child of $p$ and let $s$ denote a sub-optimal child. This means that $s^*$ is a maximizing forced node and $s$ is a maximizing choice node. The density of winning nodes from the maximizing perspective at depth $2d$ from $s$ is then given by equation (\ref{eq:even}). Since $s^*$ is a forced node, we cannot directly use equations (\ref{eq:even}) or (\ref{eq:odd}). However, we observe that all the children of $s^*$ are minimizing choice nodes, and thus, the density of winning moves from the minimizing perspective at depth $2d$ from $s^*$ is given by $f_{2d-1}$. We can use the negamax transformation to recast this as the density of winning nodes from the maximizing perspective at depth $2d$ from $s^*$: this is given by $1 - f_{2d-1}$. For large values of $d$, we know from equations (\ref{eq:even-limit}) and (\ref{eq:odd-limit}) that $f_{2d} + f_{2d-1} = 1$, or $f_{2d} = 1 - f_{2d-1}$. In other words, the density of $+1$ leaf nodes at depth $2d$ in the subtrees rooted at $s^*$ and $s$ approach the same value, for sufficiently large $d$. Since the average of the utilities of these leaves at level $2d$ will dominate the average of \emph{all} the leaves in the respective subtrees, we conclude that in large enough search trees, UCT's estimate of $\overbar{Q}(s^*)$ will approach its estimate of $\overbar{Q}(s)$. In other words, the algorithm will not be able to tell apart optimal and sub-optimal moves as it builds deeper trees, even though we have access to the true minimax value of each node in this setting, leading to the emergence of pathological behavior.
\end{proof}

% \raghu{Following is work in progress; it was not included in NeurIPS submission as messaging needs to be finessed. But once finished, it can be included in the camera-ready version.}
% We now consider the ramifications of this bound by instantiating a specific instance of a critical win-loss game.
% \begin{fact}
% \label{fact:orig}
% In a critical win-loss game with $\gamma=1$ and a branching factor $b=2$, we observe that the utilities estimated by UCT (if following a breadth-first traversal order) for the two children of the root node are within $10^{-8}$ of each other, after the search has completely explored the game tree up to depth $32$. A machine with floating-point precision less than $10^{-8}$ will thus make random decisions at the root at this point. To achieve this, we note that $N = 2^{33}-1$, and as per Theorem 1, $c \geq \sqrt{\frac{N^3}{2 \log{N}}} \approx 1.18 \times 10^{14}$.
% \end{fact}

% Admittedly, this is a somewhat unsatisfying result, as it indicates the need for an unrealistically large $c$ and a low-precision computer to ensure the emergence of lookahead pathology.

\newpage
%%%%%%%%%%%%%%%%%%%%%%%%%%%%%%%%%%%%%%%%%%%%%%%%%%%%%%%%%%%%%%%%%%%%%%

\section{Game Engine Settings}
\label{sec:config}

We collected our Chess data using the Stockfish 13 engine, with the following configuration:\\
{\footnotesize
\begin{verbatim}
"Debug Log File": "",
"Contempt": "24",
"Threads": "1",
"Hash": "16",
"Clear Hash Ponder": "false",
"MultiPV": "1",
"Skill Level": "20",
"Move Overhead": "10",
"Slow Mover": "100",
"nodestime": "0",
"UCI_Chess960": "false",
"UCI_AnalyseMode": "false",
"UCI_LimitStrength": "false",
"UCI_Elo": "1350",
"UCI_ShowWDL": "false",
"SyzygyPath": "",
"SyzygyProbeDepth": "1",
"Syzygy50MoveRule": "true",
"SyzygyProbeLimit": "7",
"Use NNUE": "false",
"EvalFile": "nn-62ef826d1a6d.nnue"

\end{verbatim}
}

\noindent We collected our Othello data using the Edax 4.4 engine\footnote{\texttt{https://github.com/abulmo/edax-reversi}} with the default settings. Edax is freely available online under a GNU GPL v3.0 license.

%%%%%%%%%%%%%%%%%%%%%%%%%%%%%%%%%%%%%%%%%%%%%%%%%%%%%%%%%%%%%%%%%%%%%%
\newpage

\section{Critical Rates in Othello}
\label{sec:othello-cr}

\begin{figure}[htb]
    \centering
    \includegraphics[width=0.70\textwidth]{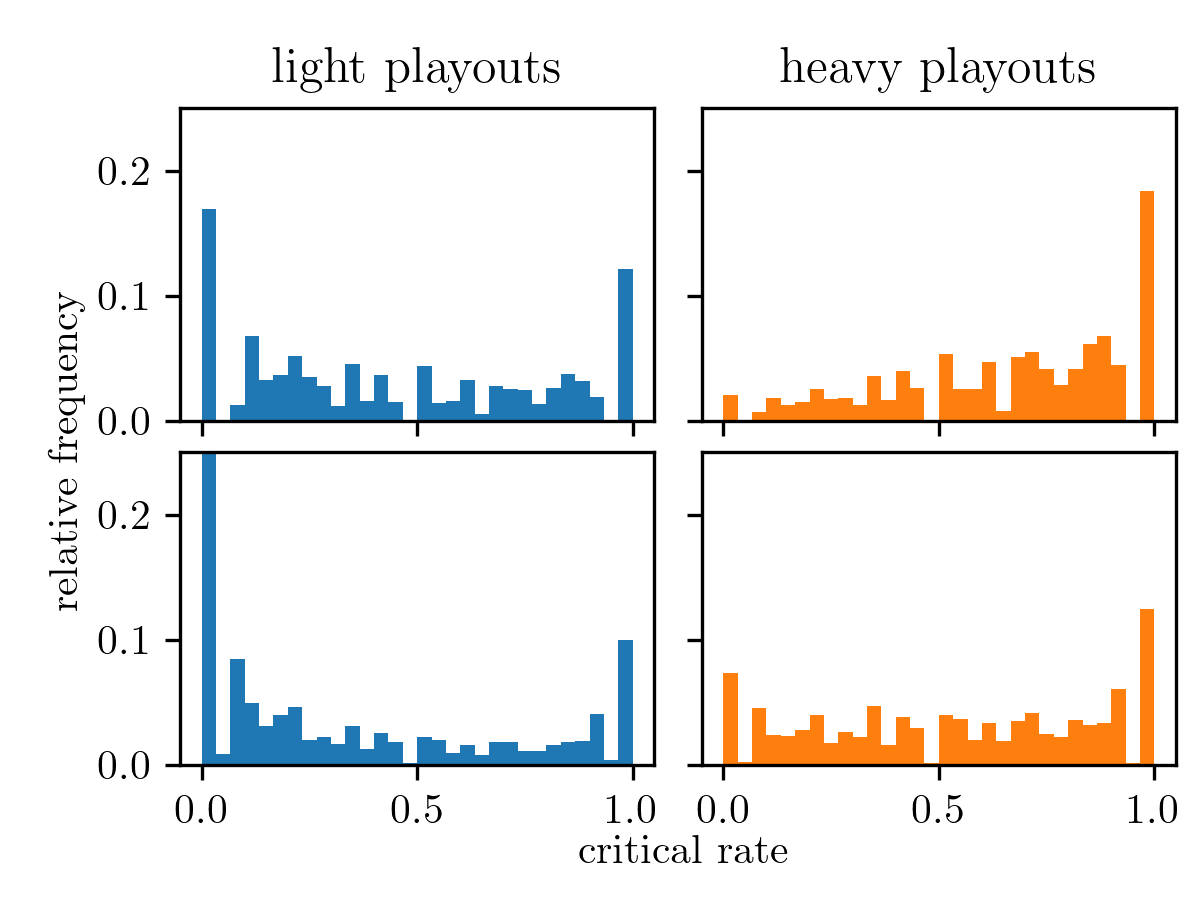}
    \caption{Histograms of empirical critical rates ($\tilde \gamma$) for Othello positions sampled $p=10$ (top row) and $p=36$ (bottom row) plies deep into the game. We sample the positions using both light playouts (left column) and heavy playouts (right column).}
\end{figure}

%%%%%%%%%%%%%%%%%%%%%%%%%%%%%%%%%%%%%%%%%%%%%%%%%%%%%%%%%%%%%%%%%%%%%%
\newpage

\section{Impact of Maximum Tree Depth on Pathology}
\label{sec:depth}

% depth 200
% left: Chess, right: Othello
\begin{figure}[htb]
    \centering
    \includegraphics[width=0.49\textwidth]{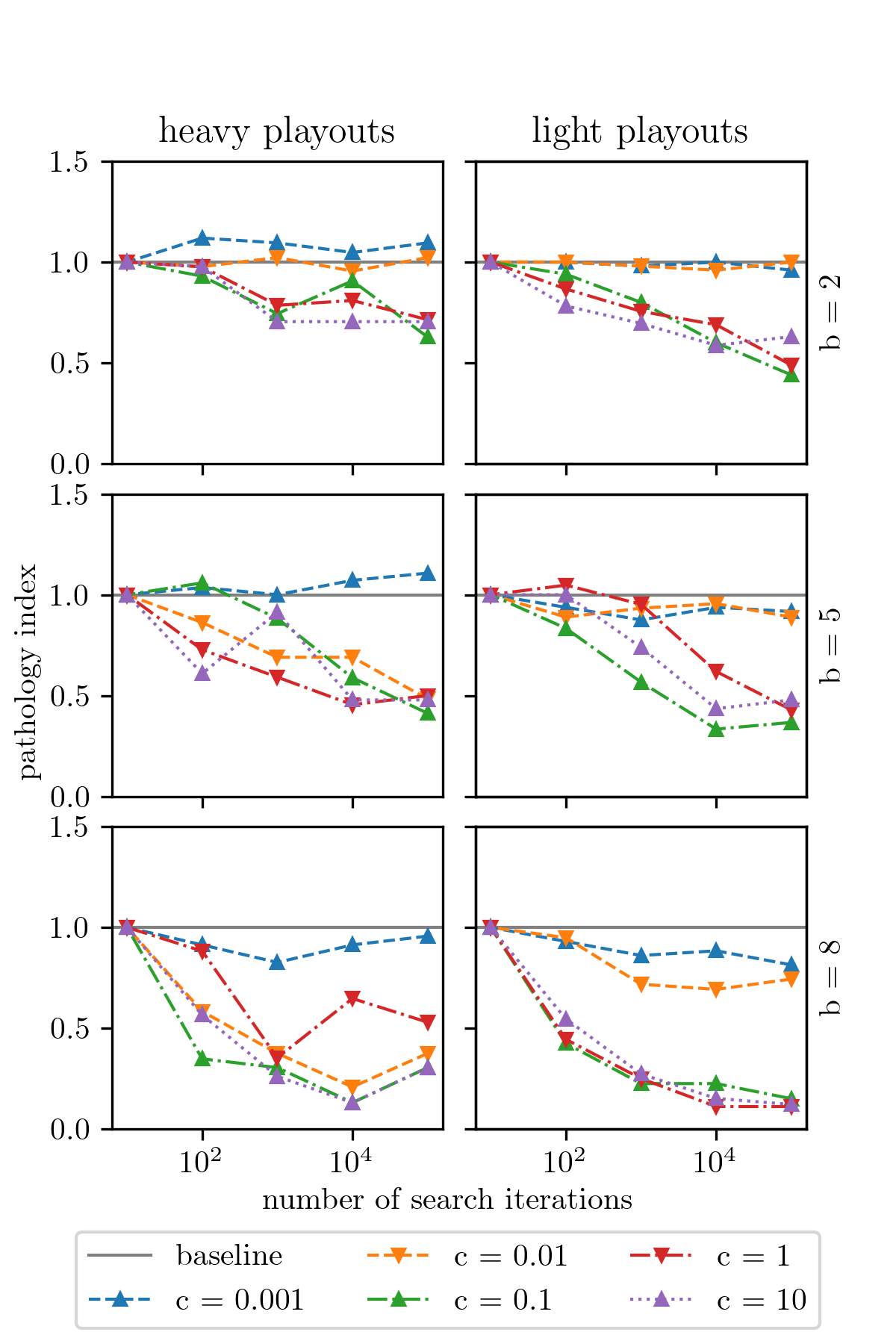}
    \includegraphics[width=0.49\textwidth]{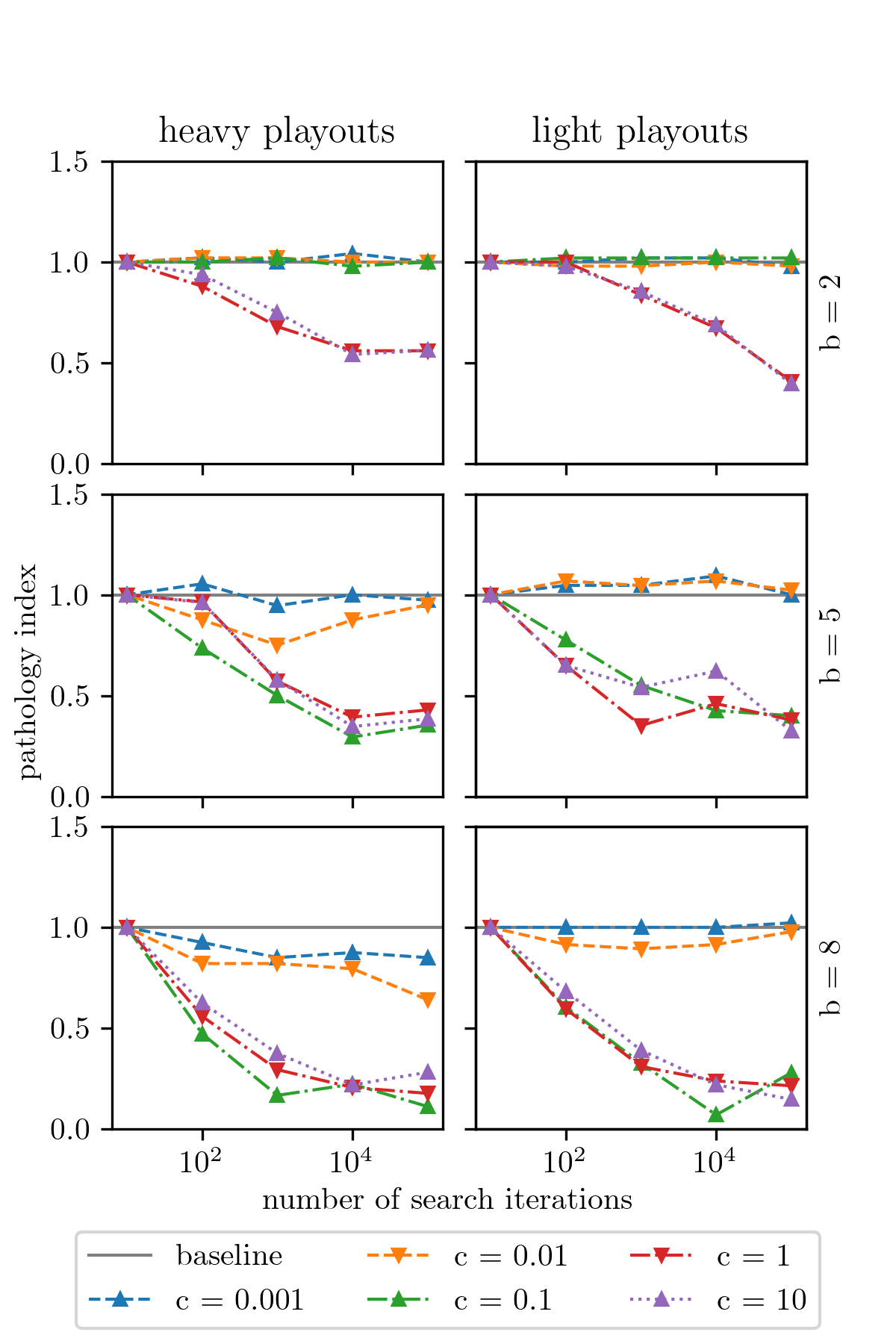}
    \caption{Measuring pathological behavior in UCT on critical win-loss games
    of depth $200$ with $\gamma=1.0$. The pair of plots on the left correspond to using a heuristic constructed from histograms of Stockfish evaluations of Chess positions sampled at depth $10$, using both light and heavy playouts. The pair of plots on the right correspond to using a heuristic constructed from histograms of Edax evaluations of Othello positions sampled at depth $10$, using both light and heavy playouts. Each colored line corresponds to an instantiation of UCT with a different exploration constant. The $x$-axis is plotted on a log-scale. We note the continued persistence of lookahead pathology.}
    \label{fig:depth-200}
\end{figure}

% depth 2000
% left: Chess, right: Othello

\begin{figure}[htb]
    \centering
    \includegraphics[width=0.49\textwidth]{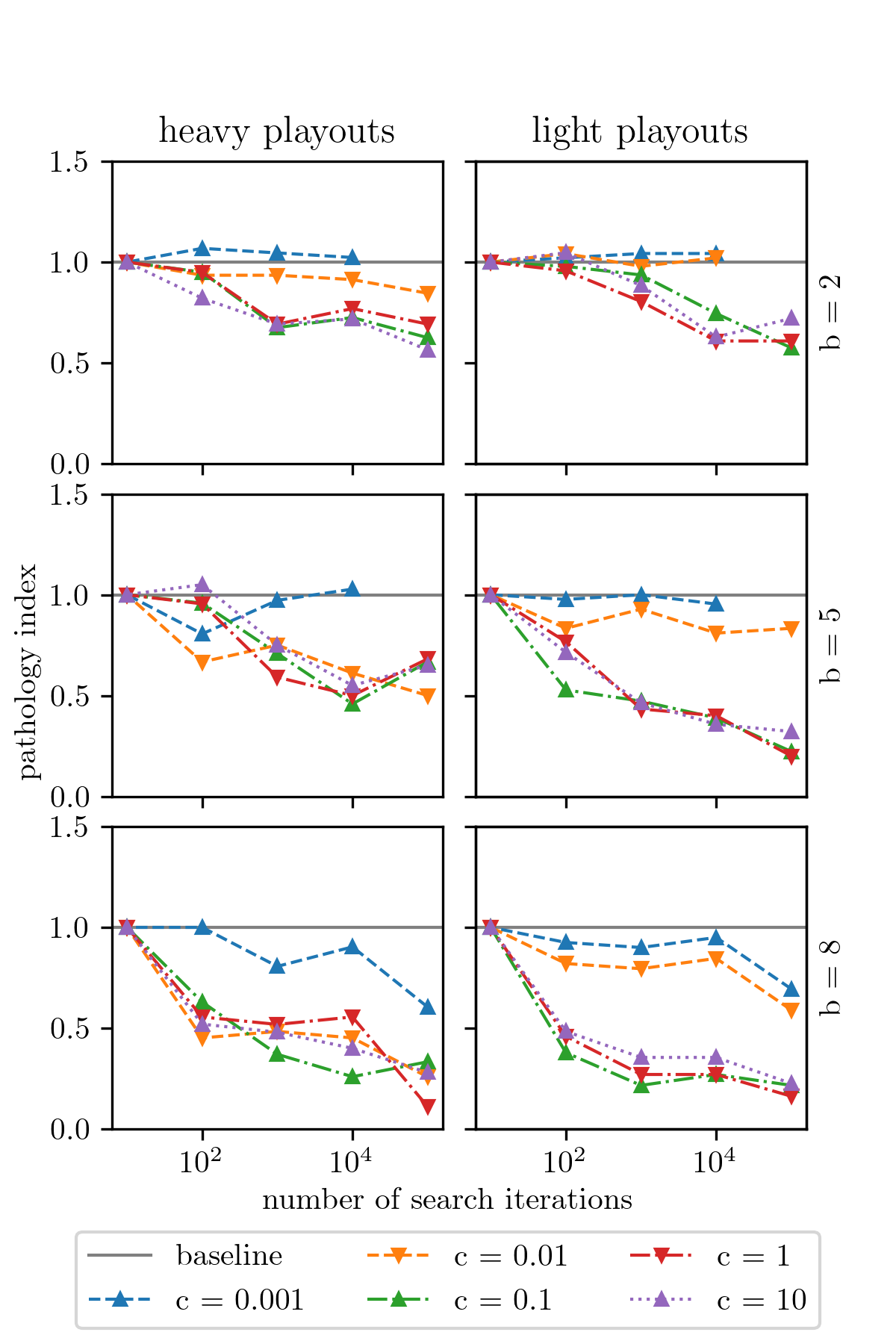}
    \includegraphics[width=0.49\textwidth]{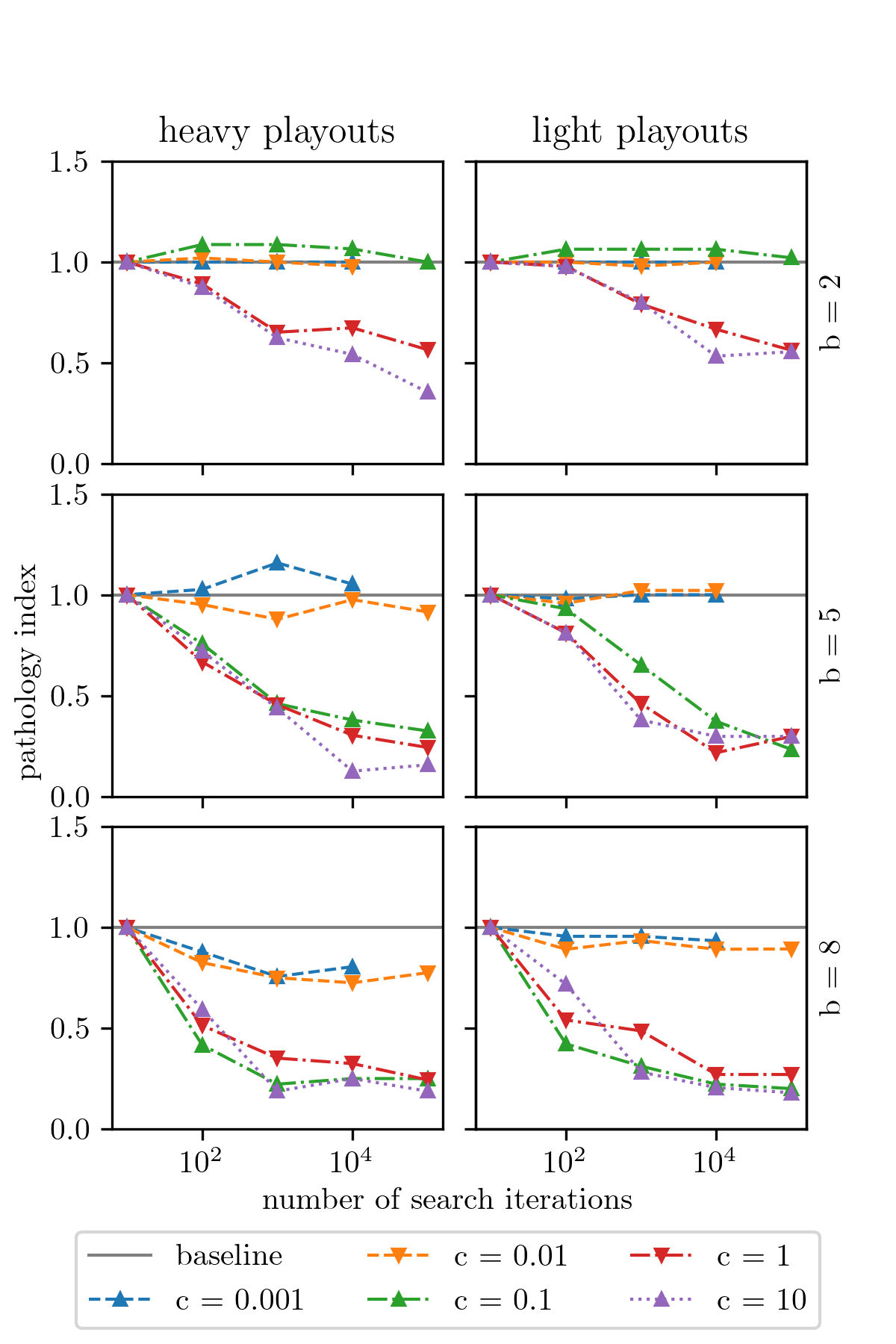}
    \caption{Measuring pathological behavior in UCT on critical win-loss games
    of depth $2000$ with $\gamma=1.0$. The pair of plots on the left correspond to using a heuristic constructed from histograms of Stockfish evaluations of Chess positions sampled at depth $10$, using both light and heavy playouts. The pair of plots on the right correspond to using a heuristic constructed from histograms of Edax evaluations of Othello positions sampled at depth $10$, using both light and heavy playouts. Each colored line corresponds to an instantiation of UCT with a different exploration constant. The $x$-axis is plotted on a log-scale. We note the continued persistence of lookahead pathology.}
    \label{fig:depth-2k}
\end{figure}

%%%%%%%%%%%%%%%%%%%%%%%%%%%%%%%%%%%%%%%%%%%%%%%%%%%%%%%%%%%%%%%%%%%%%%

\clearpage

\section{Impact of Low Critical Rate on Pathology}
\label{sec:low-cr}

\begin{figure}[htb]
    \centering
    \includegraphics[width=0.49\textwidth]{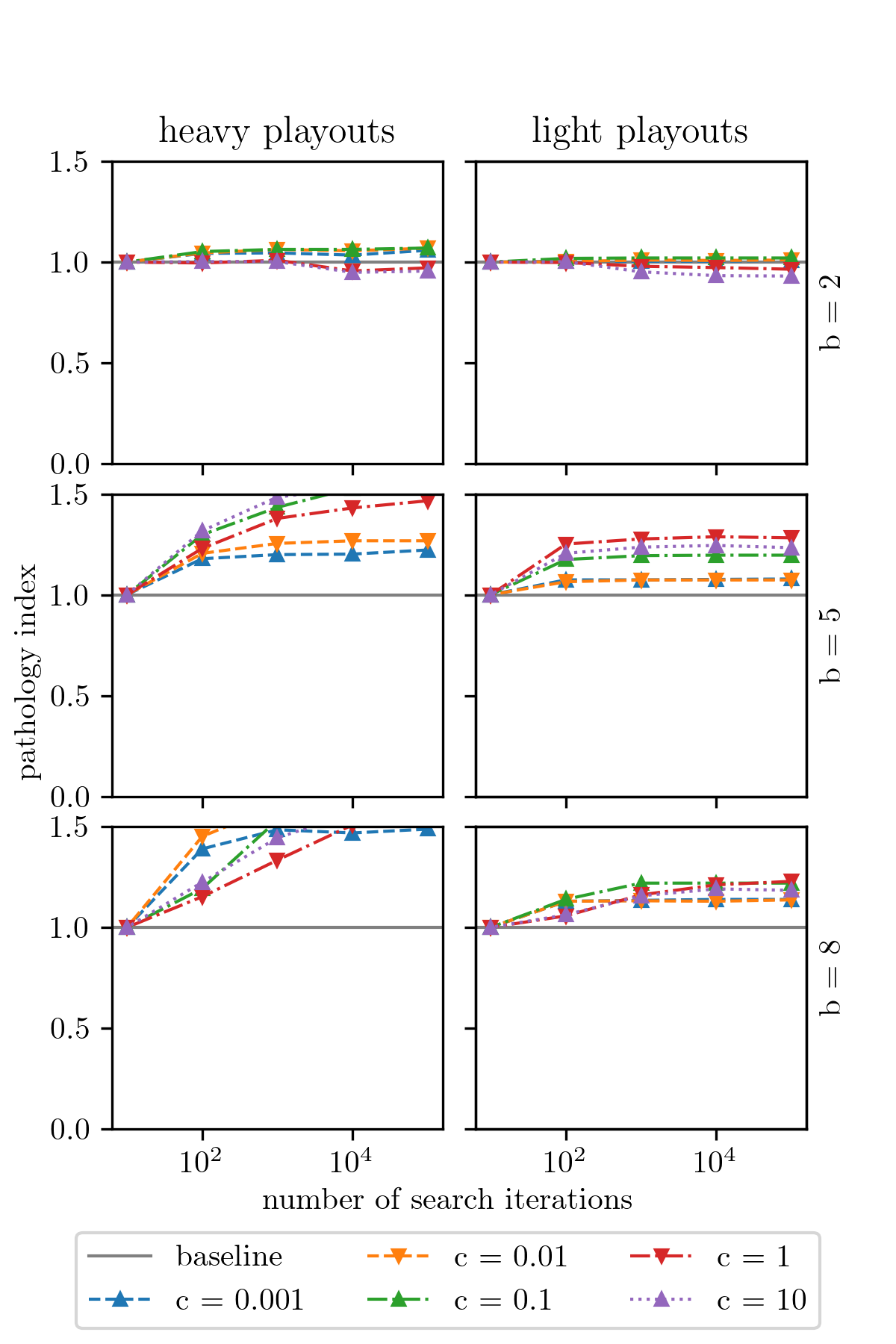}
    \includegraphics[width=0.49\textwidth]{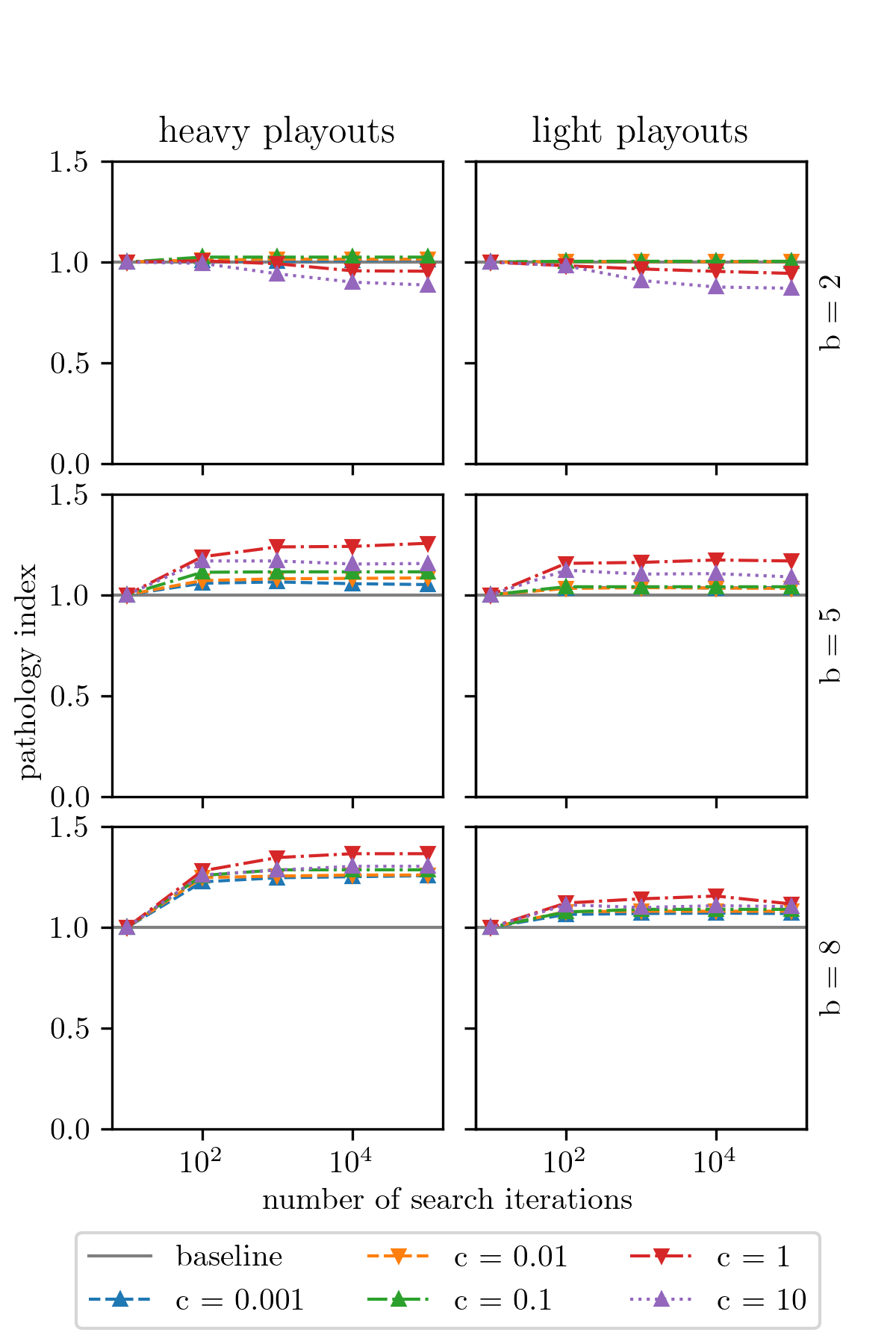}
    \caption{Measuring pathological behavior in UCT on critical win-loss games
    of depth $50$ with $\gamma=0.5$. The pair of plots on the left correspond to using a heuristic constructed from histograms of Stockfish evaluations of Chess positions sampled at depth $10$, using both light and heavy playouts. The pair of plots on the right correspond to using a heuristic constructed from histograms of Edax evaluations of Othello positions sampled at depth $10$, using both light and heavy playouts. Each colored line corresponds to an instantiation of UCT with a different exploration constant. The $x$-axis is plotted on a log-scale. We note the near complete absence of lookahead pathology in this low $\gamma$ regime.}
    \label{fig:low-cr}
\end{figure}

%%%%%%%%%%%%%%%%%%%%%%%%%%%%%%%%%%%%%%%%%%%%%%%%%%%%%%%%%%%%%%%%%%%%%%

\clearpage

\section{Investigating Pathology with Othello-Derived Heuristics}
\label{sec:othello-path}

\begin{figure}[htb]
    \centering
    \includegraphics[width=0.49\textwidth]{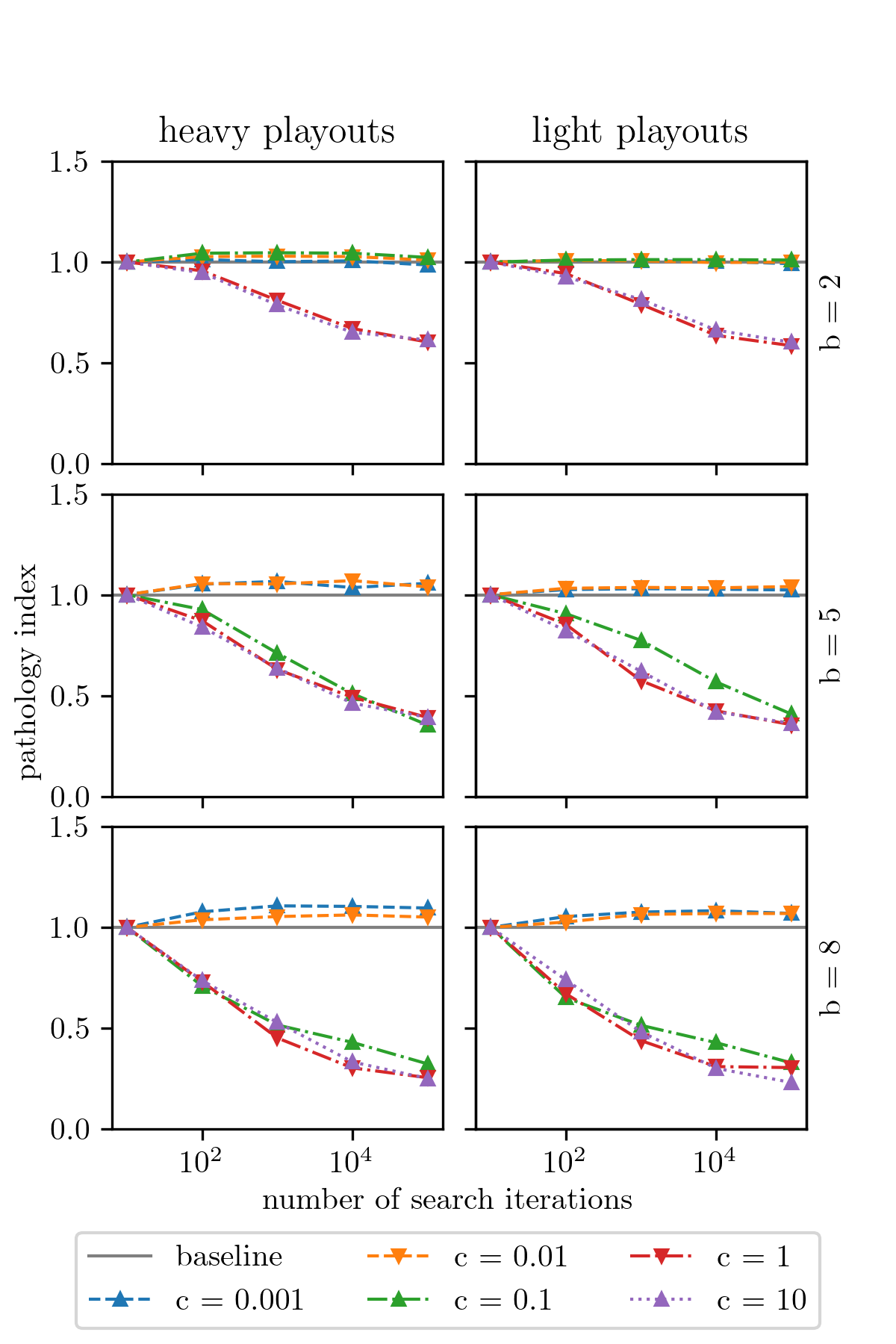}
    \includegraphics[width=0.49\textwidth]{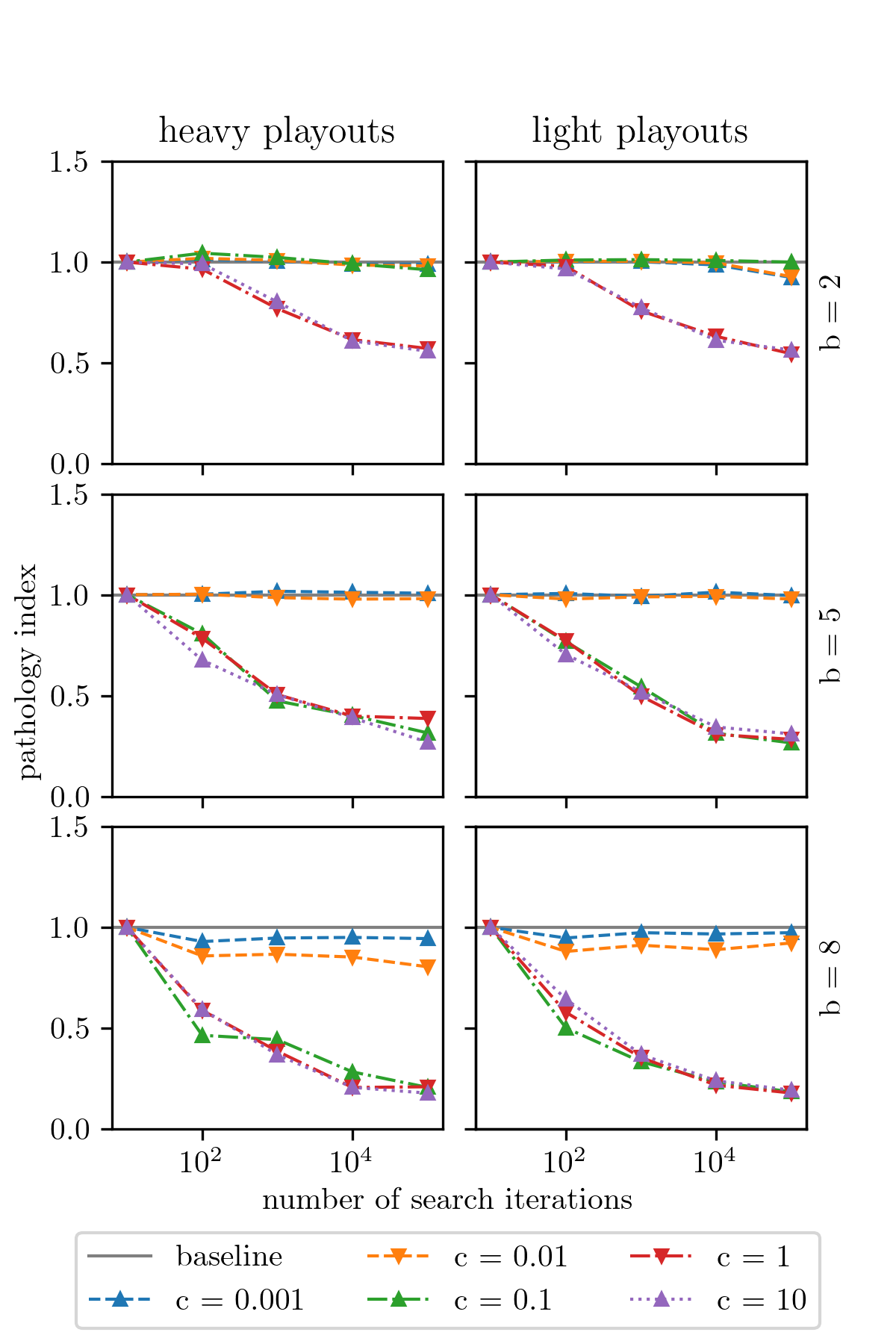}
    \caption{Measuring pathological behavior in UCT on critical win-loss games of depth $50$ with $\gamma=0.9$ (left) and $\gamma=1$ (right). The heuristic to guide UCT is constructed from histograms of Edax evaluations of Othello positions sampled at depth $10$, using both light and heavy playouts. Each colored line corresponds to an instantiation of UCT with a different exploration constant. The $x$-axis is plotted on a log-scale. We note that aside from some minor exceptions, pathological behavior generally persists even when the heuristic is sourced from a different domain.}
    \label{fig:uct-othello-path}
\end{figure}

%%%%%%%%%%%%%%%%%%%%%%%%%%%%%%%%%%%%%%%%%%%%%%%%%%%%%%%%%%%%%%%%%%%%%%
\newpage
\section{Investigating Pathology when Using True Node Utilities}
\label{sec:perfect-path}

\begin{figure}[htb]
    \centering
    \includegraphics[width=0.55\textwidth]{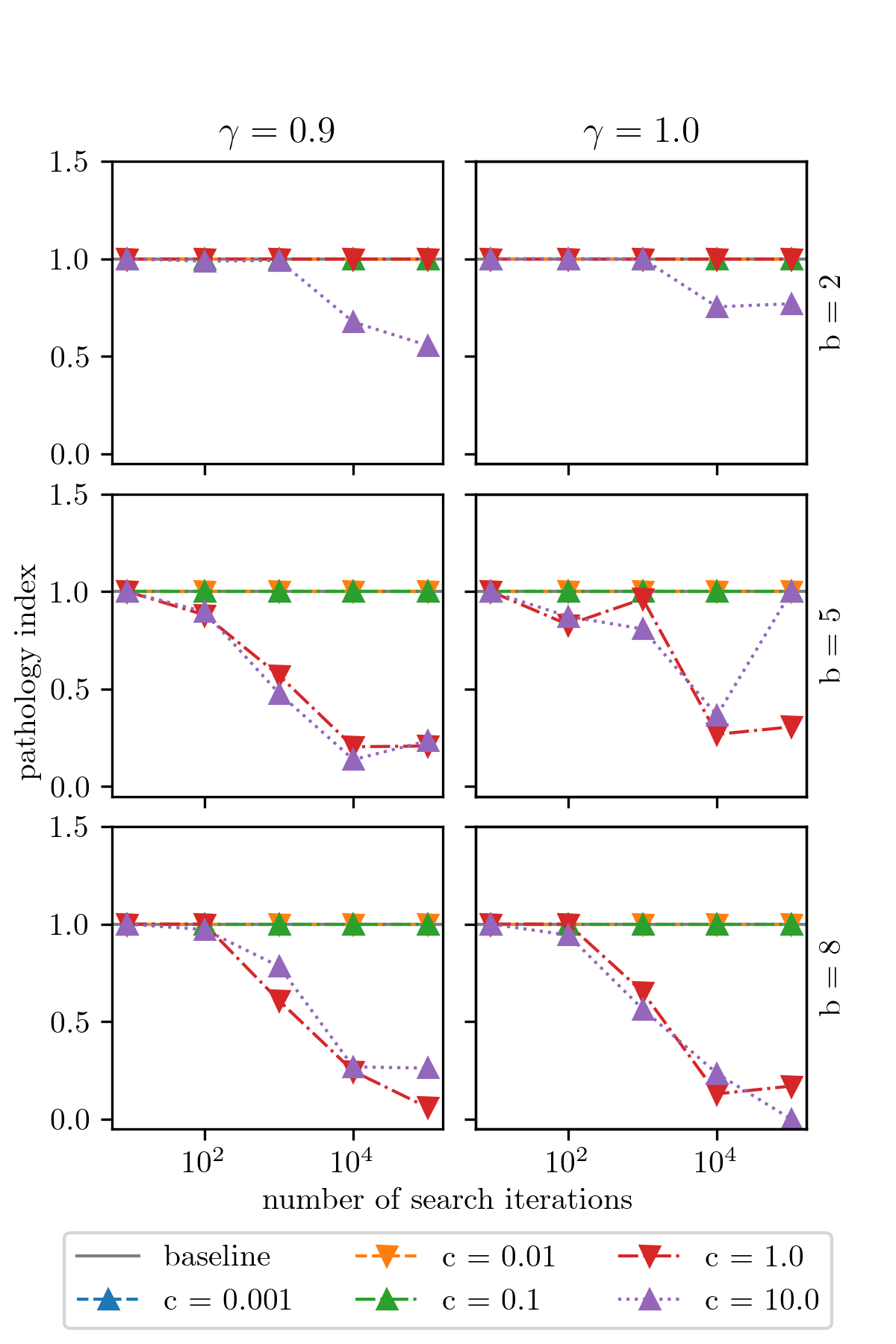}
    \caption{Measuring pathological behavior in UCT on critical win-loss games of depth $10^9$ with $\gamma=0.9$ (left) and $\gamma=1$ (right). We do not use a heuristic to guide UCT in this instance, instead relying on the true utility of each node. Each colored line corresponds to an instantiation of UCT with a different exploration constant. The $x$-axis is plotted on a log-scale. We note that lookahead pathology arises in the high-exploration regime even with access to the true game-theoretic values of nodes.}
    \label{fig:uct-perfect-path}
\end{figure}

%%%%%%%%%%%%%%%%%%%%%%%%%%%%%%%%%%%%%%%%%%%%%%%%%%%%%%%%%%%%%%%%%%%%%%

\clearpage

\section{Investigating Pathology in Alpha-Beta Search}
\label{sec:ab-path}

\begin{figure}[htb]
    \centering
    \includegraphics[width=0.55\textwidth]{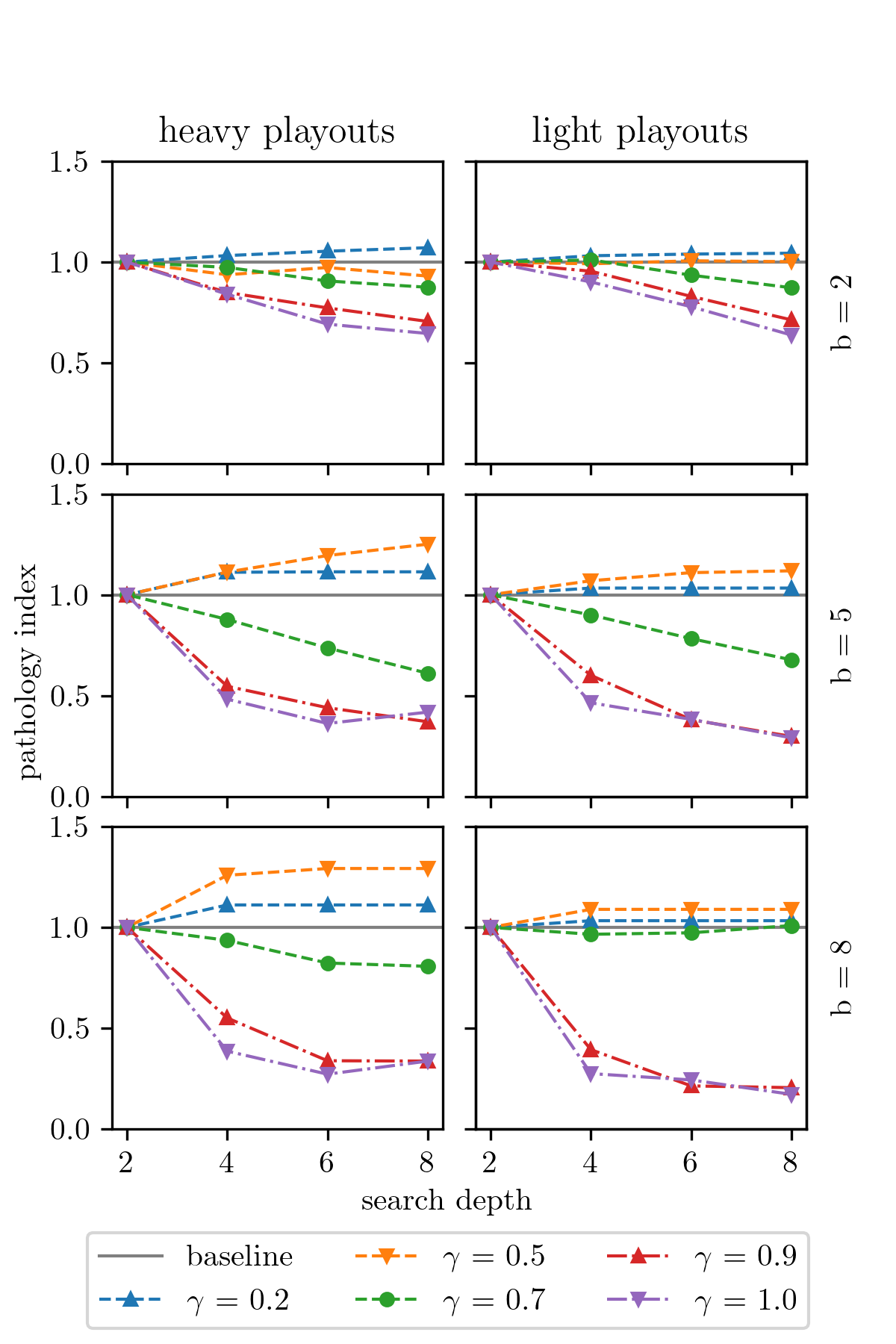}
    \caption{Measuring pathological behavior in alpha-beta search on critical win-loss games. The heuristic to guide the search constructed from histograms of Stockfish evaluations of Chess positions sampled at depth $10$, using both light and heavy playouts. The $x$-axis indicates the depth of the search tree. Each colored line corresponds to a different choice of $\gamma$. We note that pathology occurs in a wide range of parameterizations, but is most pronounced for larger branching factors and larger values of $\gamma$.}
    \label{fig:ab-path}
\end{figure}

%%%%%%%%%%%%%%%%%%%%%%%%%%%%%%%%%%%%%%%%%%%%%%%%%%%%%%%%%%%%%%%%%%%%%%

\clearpage
\section{Investigating Pathology in UCT in P-Games}
\label{sec:pgames-path}

\begin{figure}[htb]
    \centering
    \includegraphics[width=0.70\textwidth]{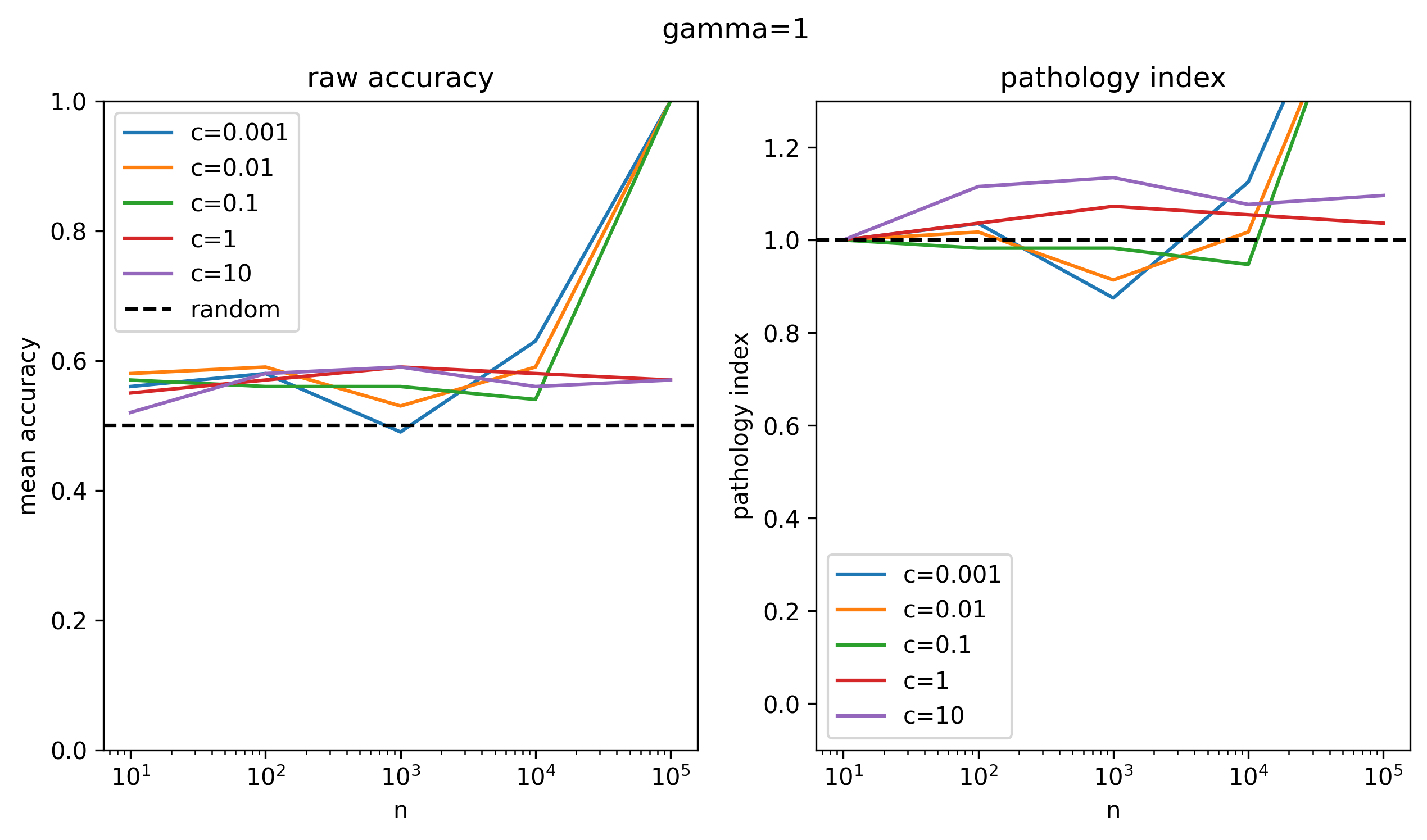}\\
    \includegraphics[width=0.70\textwidth]{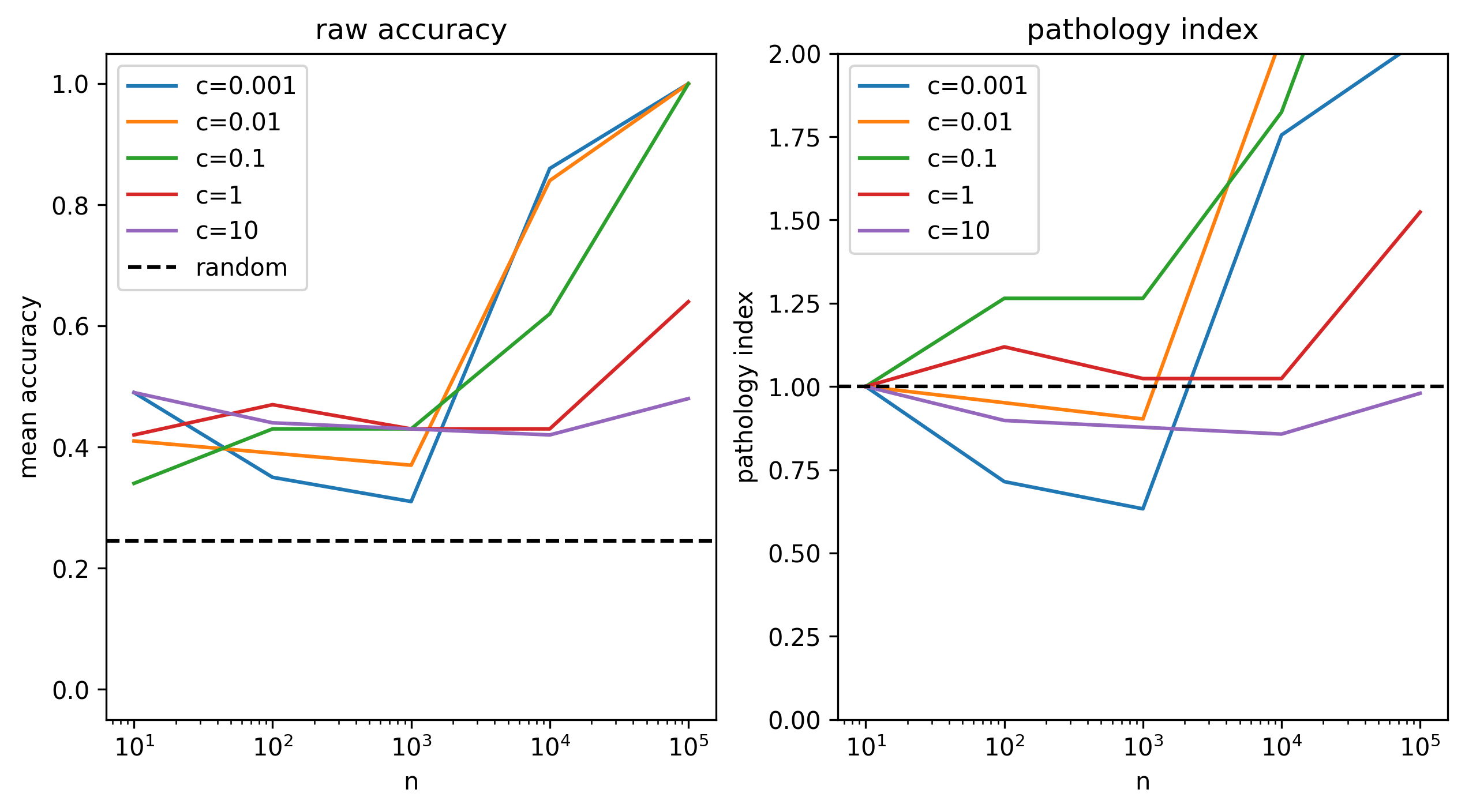}
    \caption{Measuring pathological behavior in UCT on P-games. The top panel shows results on games of depth 20 with branching factor 2. The bottom panel presents results on games of depth 8 with branching factor 5. In both parameterizations, we see that UCT exhibits little to no pathological behavior: its performance either remains (roughly) constant, or its decision accuracy improves with more search effort.} %The heuristic to guide the search constructed from histograms of Stockfish evaluations of Chess positions sampled at depth $10$, using both light and heavy playouts. The $x$-axis indicates the depth of the search tree. Each colored line corresponds to a different choice of $\gamma$. We note that pathology occurs in a wide range of parameterizations, but is most pronounced for larger branching factors and larger values of $\gamma$.}
    \label{fig:pgames-path}
\end{figure}